\documentclass{article} %
\usepackage{iclr2026_conference,times}

\usepackage{amsmath,amsfonts,bm}

\def\eqref#1{equation~\ref{#1}}

\def\1{\bm{1}}

\DeclareMathAlphabet{\mathsfit}{\encodingdefault}{\sfdefault}{m}{sl}
\SetMathAlphabet{\mathsfit}{bold}{\encodingdefault}{\sfdefault}{bx}{n}

\newcommand{\softmax}{\mathrm{softmax}}

\usepackage{hyperref}
\usepackage{url}

\usepackage{float}
\usepackage{booktabs}   %
\usepackage{multirow}   %
\usepackage{colortbl}   %
\usepackage{xcolor}     %
\usepackage{algorithm}
\usepackage{algpseudocode}
\usepackage{amsthm}
\usepackage{graphicx}
\usepackage{subcaption}
\newtheorem{theorem}{Theorem}%
\usepackage[capitalize]{cleveref}   %
\crefname{section}{Sec.}{Sec.}
\crefname{figure}{Fig.}{Figs.}
\crefname{equation}{Eq.}{Eqs.}
\crefname{appendix}{App.}{Apps.}
\crefname{table}{Tab.}{Tabs.}
\crefname{algorithm}{Alg.}{Algs.}
\crefname{theorem}{Theorem}{Theorems}

\title{Discrete Bayesian Sample Inference for Graph Generation}

\author{%
  Ole Petersen\thanks{Equal contribution}$\enspace${\normalfont\textsuperscript{1,4}}
  \quad Marcel Kollovieh\footnotemark[1]$\enspace${\normalfont\textsuperscript{1,2,3}}
  \quad Marten Lienen{\normalfont\textsuperscript{1,3}}
  \quad Stephan Günnemann{\normalfont\textsuperscript{1,2,3}}\\[10pt]
  \textsuperscript{1} School of Computation, Information and Technology, Technical University of Munich \\
  \textsuperscript{2} Munich Center for Machine Learning 
  \quad \textsuperscript{3} Munich Data Science Institute 
   \quad \textsuperscript{4} Listen Labs \\[6pt]
  Correspondence to: \texttt{m.kollovieh@tum.de}
}

\newcommand{\cat}{\mathrm{Cat}}
\newcommand{\x}{\mathbf{x}}
\newcommand{\y}{\mathbf{y}}
\newcommand{\z}{\mathbf{z}}
\newcommand{\s}{\mathbf{s}}
\newcommand{\muzero}{\ensuremath{\boldsymbol{\mu}_0}}

\newcommand{\bmu}{\ensuremath{\boldsymbol{\mu}}}

\DeclareMathOperator*{\EE}{\mathbb{E}}
\iclrfinalcopy %

\usepackage{xkcdcolors}
\hypersetup{
    colorlinks=true,
    linkcolor=xkcdDarkRose,
    urlcolor=xkcdBluish,
    linktoc=all,
    citecolor=xkcdFlatBlue,
    pdfauthor={Ole Petersen, Marcel Kollovieh, Marten Lienen, Stephan Günnemann}
}%

\begin{document}

\maketitle

\begin{abstract}
Generating graph-structured data is crucial in applications such as molecular generation, knowledge graphs, and network analysis. However, their discrete, unordered nature makes them difficult for traditional generative models, leading to the rise of discrete diffusion and flow matching models. In this work, we introduce \emph{GraphBSI}, a novel one-shot graph generative model based on Bayesian Sample Inference (BSI). Instead of evolving samples directly, GraphBSI iteratively refines a \emph{belief} over graphs in the continuous space of distribution parameters, naturally handling discrete structures. Further, we state BSI as a stochastic differential equation (SDE) and derive a noise-controlled family of SDEs that preserves the marginal distributions via an approximation of the score function. Our theoretical analysis further reveals the connection to Bayesian Flow Networks and Diffusion models. Finally, in our empirical evaluation, we demonstrate state-of-the-art performance on molecular and synthetic graph generation, outperforming existing one-shot graph generative models on the standard benchmarks Moses and GuacaMol.

\end{abstract}
\section{Introduction}
\label{sec:introduction}
Graph structures appear in various domains ranging from molecular chemistry to transportation and social networks. Generating realistic graphs enables simulation of real-world scenarios, augmenting incomplete datasets, and discovering new materials and drugs \citep{guo2022systematicsurveydeepgenerative,zhu2022surveydeepgraphgeneration}. However, their unique and complex structure poses challenges to traditional generative models that are designed for continuous data such as images. This has resulted in a diverse landscape of graph generative models, featuring autoregressive models \citep{you2018graphrnngeneratingrealisticgraphs} and one-shot models \citep{kipf2016variationalgraphautoencoders}, including a range of diffusion-based models \citep{ho2020denoisingdiffusionprobabilisticmodels}.

Recently, Bayesian Flow Networks (BFNs) \citep{graves2025bayesianflownetworks} have emerged as a novel class of models that operate on the parameters of a distribution over samples rather than on the samples themselves. This approach is particularly appealing for discrete data, as the parameters of a probability distribution evolve smoothly even when the underlying samples remain discrete. Graph generative models based on BFNs have shown competitive performance in molecule generation \citep{song2025smooth}. However, operating in parameter space and being motivated through information theory adds a layer of complexity to the BFN framework that hinders its accessibility.

\emph{Bayesian Sample Inference} (BSI) \citep{lienen2025generativemodelingbayesiansample} offers a simplified interpretation and generalizes continuous BFNs by viewing generation as a sequence of Bayesian updates that iteratively refine a belief over the unknown sample. The model is trained by optimizing its corresponding ELBO.

This work introduces \textbf{GraphBSI}, extending BSI to discrete graphs. Instead of operating on discrete states, GraphBSI evolves on the probability simplex of node and edge categories. We derive BSI for categorical data and show how to generate variably-sized graphs with it. Next, we formulate categorical BSI as an SDE and, via the Fokker–Planck equation, derive a noise-controlled family of SDEs that preserves marginals while interpolating between a deterministic probability-flow ODE and a highly stochastic sampler. Empirically, we demonstrate that GraphBSI achieves state-of-the-art results on the GuacaMol \citep{guacamol} and Moses \citep{moses} benchmarks for molecule generation. In extensive ablation studies, we show that noise control is a crucial factor for optimizing performance. An overview of our method is shown in \cref{fig:main_plot}.

Our \textbf{main contributions} can be summarized as follows:
\begin{itemize}
  \item We derive BSI for categorical data, enabling, among others, the generation of graphs and sequences. The result generalizes the Bayesian Flow Network (BFN) framework with a simplified interpretation while avoiding limit approximations in the Bayesian update.
  \item We formulate categorical BSI as an SDE. Through the Fokker-Planck equation, we derive a generalized SDE with a noise-controlling parameter and identical marginals, allowing us to interpolate between a deterministic probability flow ODE and a sampling scheme that overrides all previous predictions with the most recent one.
  \item We demonstrate that GraphBSI achieves SOTA results across most metrics in the Moses and GuacaMol molecule generation benchmarks with as few as 50 function evaluations, and further gains substantial improvements with 500 function evaluations.
\end{itemize}

\section{The Bayesian Sample Inference Framework for Graphs}
\label{sec:bsi_framework}

\begin{figure}[t]
    \centering
    \includegraphics[width=\textwidth]{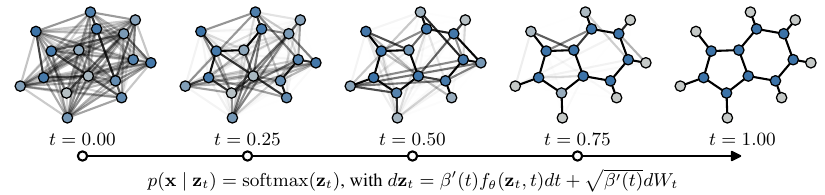}
    \caption{Illustration of GraphBSI's generative process. Nodes and edges are modeled as independent categorical variables. One edge-type is used to represent the non-existence of an edge. The latent variable $\mathbf{z}_t$ represents a \textit{distribution over graphs} rather than a graph itself. The neural network $f_\theta$ smoothly steers this distribution from a random initial distribution $\z_0$ to a distribution concentrated on valid graphs $\z_1$, which is modeled as a Stochastic Differential Equation (SDE).}
    \label{fig:main_plot}
\end{figure}

Bayesian Sample Inference (BSI)~\citep{lienen2025generativemodelingbayesiansample} is a novel generative modeling framework simplifying and generalizing Bayesian Flow Networks (BFNs)~\citep{graves2025bayesianflownetworks}. While BSI was originally presented for continuous data, we develop a theoretical framework extending BSI to categorical data analogously. We start by introducing the required background knowledge. All proofs are shown in \cref{app:proofs}.

\paragraph{Background.} Bayesian Sample Inference (BSI) \citep{lienen2025generativemodelingbayesiansample} generates samples by iteratively refining an initial belief $p(\x)$ about the sample $\x$ to be generated through noisy measurements $\y$ of $\x$. The initial belief $p(\x \mid \z_0)$ follows a broad isotropic Gaussian with parameters $\z_0=(\muzero,\sigma_0)$. The belief is then refined by a sequence of noisy measurements $\y_0,\dots,\y_{k-1}$ that follow Gaussians centered around $\x$. After receiving the measurement $\y_i$, the information contained in it is integrated into our next belief $\z_{i+1}$ through a Bayesian update. Once the belief of $\x$ is sufficiently sharp, we return a sample from it.
We train a neural network $f_\theta$ to predict the train sample $\x$ from the information collected about it in the belief $\z_i$ for each timestep $i \in {0,\dots,k-1}$. The trained neural network allows us to generate new samples during inference by creating the noisy measurements through an approximation $\hat{\x}_i=f_\theta(\z_i,i)$ of the sample $\x$ in each timestep $i$.

\paragraph{Extension to categorical data.} Now, we will focus on the case that our data lies on the simplex, i.e., we have a categorical belief for $\x$ over $c$ possible categories, i.e., $x \in \Delta_{c-1}^n \subset [0,1]^{n \times c}$. If we have access to noisy measurements $\y_i \sim \mathcal{N}(\x, \Sigma^2=\alpha_i^{-1} I)$ of the sample $\x$, we can infer $\x$ from the measurements using Bayes' theorem in a similar fashion to the continuous case. We start with an initial belief $p(\x \mid \z_0) \sim \cat(\softmax(\z_0))$, where $\z_0 \in \mathbb{R}^{n \times c}$ are the logits of a categorical distribution with $n$ independent components. Then, we can update the belief parameters $\z$ after observing $\y_i$ using Bayes' theorem.
\begin{theorem}
    \label{theorem:update_equation_v2}
    Given a prior belief $p(\x\mid \z)=\cat(\x\mid \softmax(\z))$, after observing $\y \sim \mathcal{N}(\y \mid  \mu = \x, \Sigma^2 = \alpha^{-1} \mathbf{I})$ at precision $\alpha$, the posterior belief is $p(\x\mid \z,\y,\alpha) =\cat(\x\mid \softmax(\z_\mathrm{post}))$ with
    \begin{equation}
        \z_\mathrm{post} = \z + \alpha \y.
    \end{equation}
\end{theorem}

Now, we can iterate over multiple noisy measurements and update our belief until $p(\x\mid \y_1, \dots, \y_k)$ identifies $\x$ with high probability. %
Through \cref{theorem:update_equation_v2}, we encode the information contained in all these measurements in our updated belief parameters $\z_k$ as $p(\x\mid \y_1, \dots, \y_k)=p(\x \mid \z_k) \sim \cat(\softmax(\z_k))$ with $\z_k = \z_0 + \sum_i \alpha_i \y_i$.

We process each observation $\y_i$ sequentially, inducing a notion of time. We measure $\y_i$ at time $t_i=\Delta t \cdot i \in [0,1]$ with $\Delta t = 1/(k+1)$, and the subsequent Bayesian update takes us to $t_{i+1}$. To control the total amount of information added to the belief $p(\x\mid \z_t)$ up to time $t$, we define a monotonically increasing \textit{precision schedule} $\beta\colon [0, 1] \to \mathbb{R}^+$. The measurement $\y_i$ contains the information added in the time interval $[t_i, t_{i+1}]$, and therefore we choose $\alpha_i = \beta(t_{i+1}) - \beta(t_i)$. Note that the update of the logits in \cref{theorem:update_equation_v2} is fundamentally different than that of continuous BSI. Here, the belief components accumulate in each update, whereas in the continuous case, the update is interpolated with its previous state.

\textbf{Generative model construction.} We build a generative model for categorical data given the above procedure, similarly as done for BSI with continuous data \citep{lienen2025generativemodelingbayesiansample}. We begin with a logit $\z_0$ defining the initial belief of the sample $\x$ that we will generate in the end, with $\z_0 \sim \mathcal{N}(\muzero,\beta_{0})$ sampled from a simple prior distribution. As $\x$ is unknown a priori, we cannot measure it, so instead we estimate it from the information we have gathered so far encoded in our latest belief. Let $f_\theta: \mathbb{R}^{n \times c} \times [0,1] \mapsto \Delta^n_{c-1}$ be a neural network with parameters $\theta$ estimating the unknown sample $\x$ behind our observations given our current belief $\z_t$ and time $t$. We estimate $\x$ as $\hat{\x}_i=f_\theta(\z_i,t)$, followed by a noisy measurement $\y_i \sim \mathcal{N}(\hat{\x}_i, \Sigma^2=\alpha^{-1}_i)$ centered around $\hat{\x}_i$ with precision $\alpha_i$. Then, we update our belief with $\y_i$ via \cref{theorem:update_equation_v2}. Now, we repeatedly predict $\hat{\x}_i$, measure $\y_i$, and update the belief parameters $\z_{i+1} \gets \z_i + \alpha_i \y_i$ until our belief is sufficiently sharp at $t=1$. Finally, we return a sample from $\cat(\x\mid \softmax(\z_1))$. See \cref{alg:bsi_sampling_categorical} for a formal description.

\textbf{Evidence Lower Bound.} To train our neural network, we interpret CatBSI as a hierarchical latent variable model to derive an evidence lower bound (ELBO) of the sample likelihood \citep{kingma2022autoencodingvariationalbayes}, providing a natural training target. As latent variables, we choose the beliefs $\z_0,\dots,\z_k$. Their distribution in \cref{alg:bsi_sampling_categorical} factorizes, allowing us to write

\begin{equation}
    p(\x)= \EE\limits_{p(\z_0) \prod_{i=1}^k p(\z_i|\z_{i-1},\theta)}\left [p(\x \mid \z_k) \right ].
\end{equation}

As encoding distribution $q(\z_0,\z_1,\dots,\z_k \mid \x)$, we choose the distribution induced under \cref{alg:bsi_sampling_categorical} with a fixed reconstruction $f_\theta(\z,t)=\x$. Thanks to the simple form of \cref{theorem:update_equation_v2}, it is straightforward to compute the marginal $q(\z_i\mid \x)$:

\begin{equation} \label{eq:zi_marginal}
    \z_i = \z_0 + \sum_{j=0}^{i-1} \alpha_j \y_j \sim 
    \mathcal{N}(\muzero + \beta(t_i)\x,\Sigma^2=\beta_0+\beta(t_i))
\end{equation}

Equipped with this, we can derive the following ELBO:
\begin{theorem}
    \label{theorem:bsi_elbo}
For categorical BSI, the log-likelihood of $\x$ under \cref{alg:bsi_sampling_categorical} is lower bounded by
    \begin{equation}
        \log p(\x) \geq \EE\limits_{\substack{\z_k \sim q(\z|\x,t_k)}}[\log p(\x\mid\z_k)] -\frac{k}{2}  \EE\limits_{\substack{i \sim \mathcal{U}(0,k-1) \\ \z_i \sim q(\z|\x,t_i)}}[(\beta(t_{i+1})-\beta(t_i))||f_\theta(\z_i,t_i)-\x||_2 ^2],
    \end{equation}
    where $q(\z \mid \x, t) = \mathcal{N}(\z \mid \muzero + \beta(t)\x, \beta_0 + \beta(t) I)$ and $p(\x\mid\z_k)=\cat(\x\mid \softmax(\z_k))$.
\end{theorem}
The first term does not depend on $\theta$ and therefore cannot be optimized; we only need to minimize the second term. For $k \to \infty$, we have that $k (\beta(t_{i+1})-\beta(t_i)) \to \beta'(t_i)$ since $\Delta t = t_{i+1}-t_i = 1/(k+1) \approx 1/k$, and $t_i \sim \mathcal{U}(0,1)$. Maximizing the ELBO for $k\to \infty$ over the dataset above is therefore equivalent to minimizing

\begin{equation}
    \label{eq:categorical_bsi_loss}
    \mathcal{L} \equiv \EE\limits_{\substack{\x \sim p(\x)\\t \sim \mathcal{U}(0,1) \\ \z \sim q(\z|\x, t)}}[\beta'(t)/2\cdot||f_\theta(\z,t)-\x||_2 ^2]
\end{equation}

The loss above immediately yields the training procedure \cref{alg:bsi_training_categorical}. This matches the continuous-time categorical BFN loss up to a constant when $\beta_0 \to 0$, i.e., the prior is a Dirac delta at $t=0$.
\begin{figure}
\input{includes/main_algorithms}
\end{figure}

\paragraph{Adaptation for graphs.} We represent graphs with $N$ nodes as tuples $(X, A)$, where $X \in \Delta_{c_X-1}^{N} \subset [0,1]^{N \times c_X}$ are the one-hot encoded categories of each node and $A \in \Delta_{c_A-1}^{N \times N} \subset [0,1]^{N \times N \times c_A}$ the one-hot encoded categories of each edge, with the first category denoting the absence of an edge. We treat each node and edge as an independent component of the categorical belief, allowing us to apply the categorical BSI framework to graphs. Note that dependence between edges is introduced via our network $f$. We choose a permutation invariant reconstruction network $f_\theta$, resulting in a permutation invariant generative model when the noise is isotropic.

To enable a varying number of nodes in the graph, we first sample a number of nodes $N$ from the marginal node count distribution, and subsequently generate the node and edge values. In practice, this is achieved by masking out inactive nodes and edges for train graphs with fewer nodes.

\paragraph{Adaptation for sequences} As a general discrete generative model, Categorical BSI is applicable for sequence generation, too. Here, a sequence $S$ of length $l$ with a vocabulary size $v$ is represented  in the one-hot-encoded format $S\in\Delta^l_v \subset [0,1]^{l\times v}$. We include an exemplary implementation trained on DNA sequences in \cref{sec:sequences}.

\section{Categorical BSI as a Stochastic Differential Equation}
\begin{figure}[ht]
    \centering
    \includegraphics[width=\textwidth]{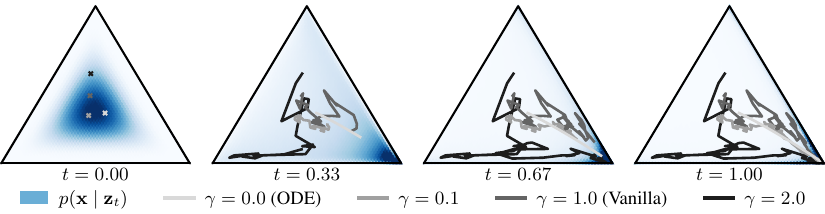}
    \caption{Trajectories of the SDE \cref{theorem:generalized_sde} for different values of $\gamma$ with three classes and fixed reconstruction $f_\theta(\z_t, t)=\hat{e}_2$. At $\gamma=0$, the sampler resembles a probability flow ODE as in flow matching. Increasing $\gamma$ leads to noisier trajectories. At $\gamma=1$, the original SDE in \cref{theorem:sde} is recovered, and increasing the noise further makes the trajectories even more volatile. The density function of the marginal distribution $p(\x\mid\z_t)$ (shown in the background) is identical for all $\gamma$.}
    \label{fig:samplers}
\end{figure}

In this section, we analyze the update equation in \cref{theorem:update_equation_v2} and take the infinite-step limit, obtaining an SDE. We then introduce a parameter that controls the stochasticity and yields a family of SDEs with identical marginals.

\paragraph{SDE Dynamics.} First, we notice that as the number of steps $k$ increases, i.e., $\Delta t := 1/(k+1) \to 0$, the updates in \cref{theorem:update_equation_v2} converge to the following SDE.
\begin{theorem}
    \label{thm:categorical_bsi_sde}
    As $\Delta t \to 0$, the update equation in \cref{theorem:update_equation_v2} converges to the following SDE:
    \begin{align}
        d\z_t & = \beta'(t) f_\theta(\z_t, t) dt + \sqrt{\beta'(t)} dW_t  \label{eq:categorical_bsi_sde}
    \end{align}
    where $dW_t$ is a Wiener process and $\z_{0}\sim \mathcal{N}(\muzero, \beta_0 \cdot I)$.
    \label{theorem:sde}
\end{theorem}
Note that while the distribution of $\z_{0}$ is not required to be normal for \cref{thm:categorical_bsi_sde} itself, it is necessary for the following steps. Phrasing the evolution of the latent variable $\z_t$ as an SDE enables the use of more advanced sampling schemes and allows us to derive a generalized SDE family. The original discrete update in \cref{theorem:update_equation_v2} is recovered by applying an Euler-Maruyama discretization of \cref{eq:categorical_bsi_sde}.

\paragraph{Generalized SDE.} We now generalize~\cref{eq:categorical_bsi_sde} to a family that preserves the marginal probability paths $p_t(\z_t)$ while controlling stochasticity via the parameter $\gamma$, similar to \cite{karras2022elucidatingdesignspacediffusionbased}:

\begin{theorem}
    The SDE in \cref{theorem:sde} is generalized by the following family of SDEs with equal marginal densities $p_t(\z_t)$:
    \begin{align}
        d\z_t & = \beta'(t) f_\theta(\z_t, t) dt + \frac{\gamma -1}{2}\beta'(t)\nabla_{\z_t}\log p_t(\z_t)dt+\sqrt{\gamma\beta'(t)} dW_t \label{eq:generalized_sde}
    \end{align}
    where $dW_t$ is a Wiener process and $\z_{0}\sim \mathcal{N}(\muzero, \beta_0 \cdot I)$.
    \label{theorem:generalized_sde}
\end{theorem}
Setting $\gamma=0$ yields a deterministic probability flow ODE, equivalent to \cite{xue2024unifyingbayesianflownetworks}. Unlike BFNs, however, CatBSI samples the prior belief $p(\z\mid t=0)$ rather than choosing a fixed prior, naturally avoiding the discontinuity around $t=0$. Further, choosing $\gamma=1$ recovers the original SDE in \cref{theorem:sde}, and larger $\gamma$ produces more stochastic trajectories. We visualize in \cref{fig:samplers,fig:binary_sampler} how varying $\gamma$ affects the dynamics for toy examples. Although the marginal distributions are equal for all $\gamma$ in theory, the empirical performance varies as $\nabla_{\z_t}\log p_t(\z_t)$ is not available in closed form. Higher stochasticity allows the model to correct errors made in previous sampling steps but requires a finer discretization (see \cref{sec:ablations}). In the limit $\gamma\to\infty$, the sampler effectively overwrites the current state completely in every step (see \cref{app:infinite_noise_sampling}).
To turn \cref{eq:generalized_sde} into a practical sampling algorithm, we approximate the score function $\nabla_{\z_t}\log p_t(\z_t)$, as described in the following.

\begin{theorem}
    \label{theorem:score_function}
    The BSI loss \cref{eq:categorical_bsi_loss} also is a score matching loss with the score model $s_\theta(\z,t)$ parameterized as
    \begin{equation}
        s_\theta(\z, t) \equiv \frac{\muzero + \beta(t)f_{\theta}(\z,t)-\z}{\beta(t) + \beta_0}  \overset{!}{\approx} \nabla_\z \log p_t(\z)
    \end{equation}
\end{theorem}

\paragraph{Discretization and integration.} As the SDE is not solvable in closed form, we resort to numerical sampling. While a simple Euler-Maruyama (EM) approach performs well on sufficiently fine time grids, we find that integrating a locally linearized SDE within each step can improve sample quality for low numbers of neural function evaluations (see \cref{sec:ablations}). More specifically, we freeze the reconstructor $\hat{\x}=f_\theta(\z_t, t)$ over the time interval $[t, t+\Delta t]$, representing an Ornstein-Uhlenbeck process. This allows us to solve the SDE analytically within this interval.

\begin{theorem}
    \label{theorem:discretized_sde}
    Fixing the prediction $\hat{\x}=f_\theta(\z_t, t)$ and the values $\beta = \beta(t+\Delta t / 2)$, $\beta' = \beta'(t+\Delta t / 2)$ in \cref{eq:generalized_sde}  in a time interval $[t, t+\Delta t]$ yields an Ornstein-Uhlenbeck (OU) process with the exact marginal

    \begin{equation}
        \z_{t + \Delta t} \sim m + (\z_t - m) e^{-\kappa \Delta t} + \sqrt{\frac{\gamma \beta'}{2 \kappa}(1 - e^{-2 \kappa \Delta t})} \cdot \mathcal{N}(0, I),
    \end{equation}
    where $\kappa = \frac{(\gamma-1)\beta'}{2(\beta_0+\beta)}$, $m=\muzero+(\beta + \beta'/\kappa)\hat{\x}$. 
\end{theorem}
Note that the OU discretization converges towards the EM scheme for $\Delta t\to 0$ (see \cref{app:equivalence_of_sampling_algorithms}).
\paragraph{Quantizing instead of sampling.} If the belief precision at $t=1$ is sufficiently sharp, the final sampling step in \cref{alg:bsi_sampling_categorical} is de facto deterministic. However, this presents an opportunity to improve sampling efficiency: In the last few steps, simply sampling from the belief would yield too noisy samples, but the belief contains enough information so that the reconstructor can make a perfect reconstruction of it (see \cref{fig:missclassification_rate_by_t}). Therefore, we can instead stop at a lower final precision and return reconstruction projected on the sample space through a quantization. Employing the discretization schemes yields \cref{alg:generalized_sde_sampling:euler_maruyama,alg:generalized_sde_sampling:ornstein_uhlenbeck}.

We also allow a nonuniform time grid. Following \citet{karras2022elucidatingdesignspacediffusionbased}, we introduce a parameter $\rho$ that controls the distribution of function evaluations over the time grid:
\begin{equation}
  t_i = \left(\frac{i}{k}\right)^{\rho}, \qquad i=0,1,\dots,k.\label{eq:rho}
\end{equation}
Here, $\rho=1$ recovers a uniform grid; larger $\rho$ concentrates steps near the beginning ($t\approx 0$), whereas smaller $\rho$ concentrates them near the end ($t\approx 1$).
\begin{figure}
\setlength{\topsep}{0pt}
\setlength{\intextsep}{0pt}
\begin{minipage}{0.48\textwidth}
    \begin{algorithm}[H]
        \caption{Euler-Maruyama Sampling}
        \label{alg:generalized_sde_sampling:euler_maruyama}
        \begin{algorithmic}
            \Require reconstructor $f_\theta$, discretization $\Delta t$, precision schedule $\beta: [0, 1] \to \mathbb{R}^+$, $\gamma \geq 0$
            \State $\z \sim \mathcal{N}(\muzero, \beta_0 I)$
            \For{$t = 0,\Delta t,2\Delta t, \dots ,1-\Delta t$}
            \State $\hat{\x} \gets f_\theta(\z, t)$
            \State $\s_\theta \gets \frac{\muzero + \beta(t)\hat{\x}-\z}{\beta(t) + \beta_0}$
            \State $\bmu \gets \beta'(t)(\hat{\x} + \frac{\gamma-1}{2} \s_\theta)$
            \State $\sigma \gets \sqrt{\gamma \beta'(t)}$
            \State $\z \gets \z + \bmu \Delta t + \sigma \sqrt{\Delta t} \cdot \mathcal{N}(0, I)$
            \EndFor
            \State \Return $\mathrm{Quantize}(f_\theta(\z,t=1))$
        \end{algorithmic}
    \end{algorithm}
\end{minipage}
\hfill
\setlength{\topsep}{0pt}
\setlength{\intextsep}{0pt}
\begin{minipage}{0.48\textwidth}
    \begin{algorithm}[H]
        \caption{Ornstein-Uhlenbeck Sampling}
        \label{alg:generalized_sde_sampling:ornstein_uhlenbeck}
        \begin{algorithmic}
            \Require reconstructor $f_\theta$, discretization $\Delta t$, precision schedule $\beta: [0, 1] \to \mathbb{R}^+$, $\gamma > 1$
            \State $\z \sim \mathcal{N}(\muzero, \beta_0 I)$
            \For{$t = \Delta t/2,\Delta t+\Delta t/2, \dots ,1-\Delta t/2$}
            \State $\hat{\x} \gets f_\theta(\z, t)$
            \State $\kappa \gets \frac{(\gamma-1)\beta'(t)}{2(\beta_0+\beta(t))}$
            \State $m\gets \muzero+(\beta(t) + \beta'(t)/\kappa)\hat{\x}$
            \State $\sigma^2\gets\frac{\gamma \beta'(t)}{2 \kappa}(1 - e^{-2 \kappa \Delta t})$
            \State $\z \gets m + (\z - m) e^{-\kappa \Delta t} + \sqrt{\sigma^2} \cdot \mathcal{N}(0, I)$
            \EndFor
            \State \Return $\mathrm{Quantize}(f_\theta(\z,t=1))$
        \end{algorithmic}
    \end{algorithm}
\end{minipage}
\end{figure}

\section{Experiments}
\label{sec:experiments}

In this section, we present our empirical results. We benchmark our model against state-of-the-art baselines from the diffusion and flow-matching literature on unconditional molecular and synthetic graph generation. The GuacaMol and Moses benchmarks for molecular generation \citep{guacamol,moses} serve as our primary evaluation datasets. Additionally, we conduct ablation studies to analyze the impact of various components and hyperparameters on the model's performance. Further, we report results on the synthetic planar, tree, and stochastic block model graph generation tasks \citep{bergmeister2024efficientscalablegraphgeneration,martinkus2022spectrespectralconditioninghelps}.
\subsection{Experimental Setup}
\paragraph{Datasets.} To test performance on real-world graphs, we train GraphBSI on the Moses \citep{moses} and GuacaMol \citep{guacamol} datasets for molecular generation. We extract graphs out of the dataset smiles with RDKit \citet{rdkit} and construct the node features $X$ and adjacency matrix $A$ in the format described in \cref{sec:bsi_framework}, where atom- and bond types correspond to node- and edge categories, respectively. Further, we include results for the planar, tree, and stochastic block model \citep{martinkus2022spectrespectralconditioninghelps,bergmeister2024efficientscalablegraphgeneration} synthetic graph generation datasets. Find a summary in \cref{tab:datasets}.
\paragraph{Evaluation metrics.} We follow the standard evaluation practices as established by \citet{moses,guacamol,fcd_paper} for molecule generation and \cite{martinkus2022spectrespectralconditioninghelps,bergmeister2024efficientscalablegraphgeneration} for synthetic graph generation. Find a detailed description in \cref{tab:molecular_metrics,tab:synthetic_graph_metrics}.

\paragraph{Practical considerations.} The reconstruction network $f_\theta$ is parameterized using the same graph transformer architecture as \citet{qin2025defogdiscreteflowmatching,vignac2023digressdiscretedenoisingdiffusion}, with the node- and edge logits and class probabilities, entropy, random walk features, and sinusoidal embeddings~\citep{vaswani2017attention} of the timestep $t$ with frequencies proposed by \citet{lienen2024zero} as features. Empirically, we find that an exponential precision schedule with a final precision that allows for a near-perfect reconstruction maximizes performance (see \cref{tab:hyperparameters,fig:final_precision_ablations,fig:missclassification_rate_by_t}). For both latent node- and edge classes, we choose a normal prior with the marginal distribution over the dataset and a small variance of $1.0$. Finally, we apply a preconditioning scheme where the neural network predicts the difference between the belief and the true sample, setting $f_\theta(z,t) = \softmax(z + \tilde{f}_\theta(z,t))$.

\textbf{Evaluation} After training to convergence, we evaluate the benchmark metrics for both discretization schemes \cref{alg:generalized_sde_sampling:euler_maruyama,alg:generalized_sde_sampling:ornstein_uhlenbeck}. For both molecule generation benchmarks, we report results with a compute budget of 50 and 500 discretization steps. In each of the four configurations (2 discretization schemes, 2 numbers of steps), we optimize the noise level $\gamma$ and report the best result. Find the final configurations in \cref{tab:hyperparameters}. 
For the synthetic graph generation benchmarks, we report results with the best-performing noise level and the Ornstein-Uhlenbeck discretization with 1000 function evaluations.
\subsection{Results}
\paragraph{Molecule Generation.} As illustrated in \cref{tab:molecule_results}, GraphBSI is competitive with 50 steps with both discretization schemes for both molecule benchmarks, achieving state-of-the-art results on the majority of the metrics. Notably, GraphBSI outperforms DeFoG with both discretization schemes on all metrics except novelty on Moses. On most metrics, the OU discretization performs better than the EM scheme.
At the full 500 steps, GraphBSI with the OU discretization outperforms all existing models on all metrics on GuacaMol, saturating validity and consistently exceeding the state-of-the-art. The EM scheme performs slightly worse than OU on most metrics, but remarkably surpasses the state-of-the-art on the FCD metric, reducing it from $1.07$ to $0.72$ on Moses. Find an extended comparison in \cref{tab:molecule_results_extended}.
\begin{table}[h]
  \centering
  \vspace{-0.75em}
  \caption{Results on the GuacaMol and Moses benchmarks for molecular generation with 50 and 500 sampling steps and the Euler-Maruyama (EM) and Ornstein-Uhlenbeck (OU) discretization.}\label{tab:molecule_results}
  \includegraphics[width=\textwidth]{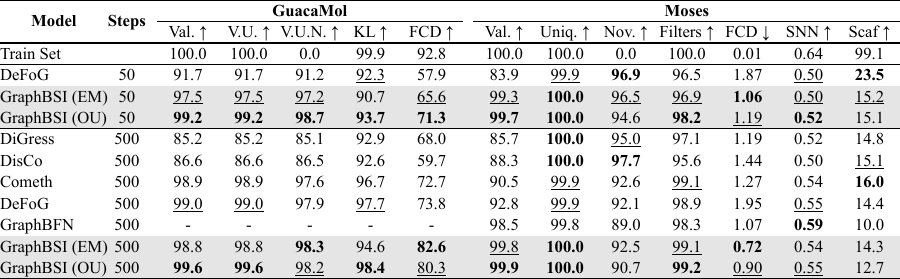}
\end{table}

\paragraph{Synthetic Benchmarks.} As shown in \cref{tab:synthetic_graph_results}, GraphBSI achieves competitive results on the synthetic graph generation benchmarks. Our model saturates validity on the planar- and tree graph generation tasks, and achieves adequate validity on the stochastic block model graphs. The mean ratio as a measure of distribution similarity is competitive on all three datasets, even though the metric should be taken with a grain of salt due to the small dataset size of only $128$ graphs, resulting in high uncertainty in the evaluation.
\begin{table}[h]
  \centering
  \caption{Results on the synthetic graph generation benchmarks. Like DeFoG, we generate 40 graphs five times and report the mean and standard deviation over the runs.}\label{tab:synthetic_graph_results}
  \includegraphics[width=\textwidth]{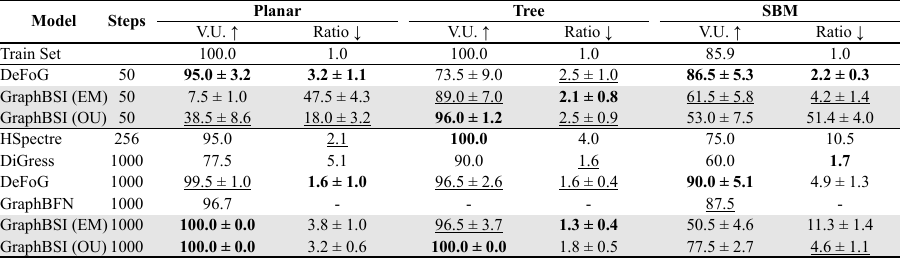}
\end{table}
\subsection{Ablation Studies}\label{sec:ablations}
\begin{figure}
    \centering
    \includegraphics[width=\textwidth]{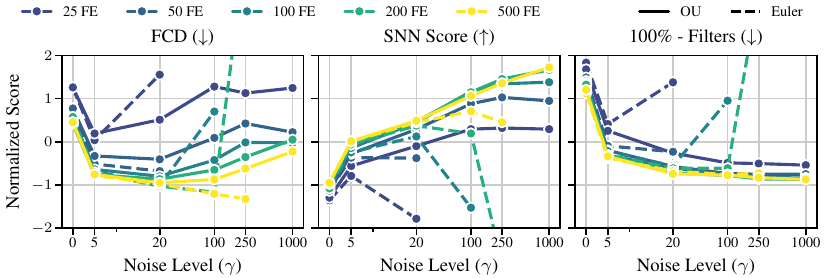}
    \caption{Normalized metrics (zero mean, unit variance) vs. noise level $\gamma$ for different numbers of function evaluations (FE) and discretization schemes. Our custom Ornstein-Uhlenbeck discretization scheme is denoted as OU, while the standard Euler-Maruyama scheme is written as Euler. Some values for the Euler scheme are missing since the sampler becomes unstable if $\gamma \cdot \Delta t$ becomes too large (see \cref{app:euler_maruyama_stability}).}
    \label{fig:noise_ablation}
\end{figure}

\paragraph{Noise level.} To test the effect of the compute budget, noise level, and discretization scheme on performance, we conduct a grid search over the number of function evaluations (NFEs) in $\{25,50,100,200,500\}$, noise levels $\gamma$ in $\{0.0,5.0,20.0,100.0,250.0,1000.0\}$, and both discretization schemes on the Moses dataset. As shown in \cref{fig:noise_ablation}, performance in both discretization schemes is closely related at low noise levels, which is to be expected since both discretize the same SDE. Higher compute budgets lead to better performance. However, the Euler-Maruyama scheme becomes unstable at higher noise levels, leading to a significant drop in performance (see \cref{app:euler_maruyama_stability}). In contrast, the Ornstein-Uhlenbeck scheme remains stable, and both the SNN score and Filters metric benefit from higher noise levels. The FCD metric is optimal at a medium noise level between $20$ and $100$. With a few exceptions, the Ornstein-Uhlenbeck scheme matches or outperforms the Euler-Maruyama scheme at all compute budgets and noise levels. Novelty suffers from increased noise levels and compute budgets, which is consistent with the model generating samples closer to the training data distribution. Notably, all metrics perform poorly at a noise level of $0.0$, which corresponds to the probability flow ODE (equivalent to \citet{xue2024unifyingbayesianflownetworks}). \cref{fig:noise_ablation_relative} illustrates that optimizing the noise level is a key driver in the performance gains of our model.

\begin{figure}[ht]
    \centering
\includegraphics{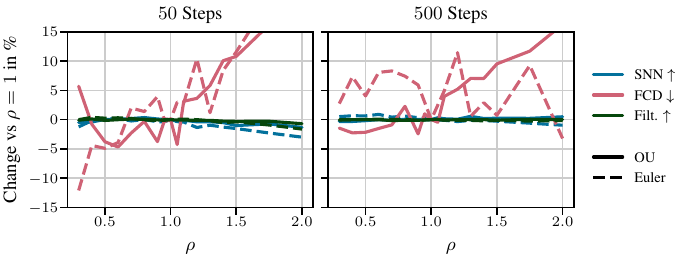}
\vspace{-0.5em}
    \caption{Performance change for changes in the non-uniform timestepping parameter $\rho$ in $t_i = (i/k)^\rho$ for $i=0,1,\dots,k$ compared to the uniform case $\rho=1$. $\rho<1$ results in a finer discretization at later timesteps, while $\rho>1$ corresponds to finer discretization at earlier steps.}
    \vspace{-0.5em}
 \label{fig:rho_ablation}
\end{figure}

\paragraph{Non-uniform timesteps.} To test whether a fine discretization is more important at some timesteps compared to others, we analyze the effect of non-uniform timestepping, putting a finer discretization at either earlier or later timesteps. As shown in \cref{fig:rho_ablation}, SNN and Moses Filters remain mostly unaffected by the choice of $\rho$; only the FCD displays a clear trend. A finer discretization at later timesteps ($\rho<1$) improves the FCD at $50$ function evaluations in both discretization schemes and at $500$ evaluations in the Ornstein-Uhlenbeck scheme.
\paragraph{Precision schedule.} We find that while an exponential precision schedule yields the best results, the difference compared to a simple linear schedule is negligible (see \cref{tab:molecule_results_extended}). One parameter that significantly affects performance is the final precision $\beta(t=1)$. As illustrated in \cref{fig:final_precision_ablations}, an excessively large final precision wastes sampler iterations in the final steps, and a too small final precision results in noisy samples. Ideally, the reconstructor is just able to predict the train samples flawlessly at $\beta(t=1)$. Finally, we isolate the effect of sampling the belief at $t=0$ instead of taking a fixed value, as with BFNs, by training a new model with a smaller initial variance of $\beta_0=0.05$, compared to the standard $\beta_0=1.0$. \cref{tab:molecule_results_extended} shows that for both values of $\beta_0$, the OU sampler outperforms the Flowback \citep{song2025smooth} sampler on most metrics. Surprisingly, the performance of the Flowback sampler drops significantly when $\beta_0$ is increased, while a higher value of $\beta_0$ improves performance for the OU sampler.

We conclude that two key factors are crucial for the performance gains of GraphBSI: First, the noise control, and second, a final precision that is just high enough for a perfect reconstruction. The exact precision schedule and non-uniform time-stepping show only a marginal contribution.

\section{Related Work}
\label{sec:related_work}

Graph generation presents three main challenges compared to image and text generation:  
(1) graphs are discrete structures, unlike images, which are continuous;  
(2) graphs have a variable shape, with both the number and arrangement of nodes and edges changing across samples, unlike the fixed dimensions of images; and  
(3) nodes in graphs lack a natural order, in contrast to text, where tokens follow a well-defined sequence. Various approaches have been proposed to tackle these challenges.

\textbf{Autoregressive models} have proven successful in text generation by sequentially predicting the next token based on previous ones \citep{brown2020few}. Applied to graphs, these models generate nodes and edges one by one, maintaining the graph structure as they proceed. This approach has been used for tasks such as molecule and social-network generation \citep{you2018graphrnngeneratingrealisticgraphs, liao2020efficientgraphgenerationgraph}. However, autoregressive models violate permutation invariance by relying on a specific node ordering.

\textbf{One-shot models} address the ordering challenge by generating the entire graph in a single step, without relying on a specific node ordering. Examples include Variational Autoencoders \citep{kingma2022autoencodingvariationalbayes}, GANs \citep{decao2022molganimplicitgenerativemodel}, normalizing flows \citep{liu2019graphnormalizingflows}, and discrete flow matching \citep{gat2024discreteflowmatching, qin2025defogdiscreteflowmatching}.

\textbf{Diffusion models} have emerged as a powerful class of one-shot generative models for continuous data such as images \citep{sohldickstein2015deepunsupervisedlearningusing, ho2020denoisingdiffusionprobabilisticmodels}. Their core idea is to learn a generative process that gradually transforms noise into clean data by reversing a diffusion process with a neural network. Noise is typically applied independently to each pixel in images or to each node in graphs, naturally resulting in a permutation-invariant model when combined with a Graph Neural Network (GNN) \citep{niu2020permutationinvariantgraphgeneration}. A variable number of nodes can be handled by conditioning the diffusion process on the node count, e.g., by first sampling a node mask and then applying diffusion to the masked graph \citep{niu2020permutationinvariantgraphgeneration, qin2025defogdiscreteflowmatching}. To improve scalability, hybrid methods that reverse a coarsening process and generate local structures with a diffusion model have also been proposed \citep{bergmeister2024efficientscalablegraphgeneration}.

\textbf{Discrete diffusion} addresses the discreteness of graphs. The most straightforward approach relaxes discrete data to a continuous space, applies diffusion, and quantizes the generated outputs back to the discrete space in a final step \citep{niu2020permutationinvariantgraphgeneration, jo2022scorebasedgenerativemodelinggraphs, jo2024graphgenerationdiffusionmixture}. Alternatively, one can use discrete diffusion in which the state is perturbed via a Markovian transition matrix in discrete time steps (often including an absorbing state) \citep{austin2023structureddenoisingdiffusionmodels}; this has been applied to graphs \citep{vignac2023digressdiscretedenoisingdiffusion, haefeli2023diffusionmodelsgraphsbenefit}. A related recent approach uses a continuous-time Markov chain for the discrete diffusion process (see \citep{campbell2022continuoustimeframeworkdiscrete}), which allows more flexible sampling on graphs \citep{siraudin2024comethcontinuoustimediscretestategraph, xu2024discretestatecontinuoustimediffusiongraph}.

\textbf{Bayesian Flow Networks} \citet{graves2025bayesianflownetworks} propose a conceptually distinct approach to discrete generative models: diffusion is applied to the \textit{parameters of a distribution over samples} rather than to the samples themselves. BFNs can be interpreted as an SDE, enabling more efficient sampling algorithms \citep{xue2024unifyingbayesianflownetworks}. This provides a solid theoretical foundation for diffusion on discrete data while retaining the benefits of smooth parameter changes, and it achieves competitive performance on protein and graph generation \citep{Atkinson2025, song2025smooth, 2025chembfn}. The flexible design of BFNs also permits joint generation of continuous and discrete quantities, for example the 3D positions, atom types, and charges in molecular generation \citep{song2024unified}.  

\textbf{Bayesian Sample Inference} \citet{lienen2025generativemodelingbayesiansample} extends BFNs by adding a prior over the distribution parameters and offers a simplified interpretation for the continuous-data case. \citet{kollovieh2025treegen} used the BSI framework to derive their generative model for hierarchies. However, they do not generalize the framework, i.e., do not derive SDE-based sampling algorithms, and do not optimize an ELBO as they specifically focus on hierarchy generation.

\section{Conclusion}
\label{sec:conclusion}
In this work, we introduce \textbf{GraphBSI}, a novel generative model for graphs based on Bayesian Sample Inference with state-of-the-art performance in large molecule generation benchmarks. Similar to Bayesian Flow Networks, GraphBSI iteratively refines a belief over the graph structure, modeled as a categorical distribution over adjacency matrices, through Bayesian updates. We show that in the limit of infinitesimal time steps, GraphBSI converges to a Stochastic Differential Equation (SDE). Further, we employ the Fokker-Planck equation to derive a generalized SDE with a tunable noise parameter, allowing us to interpolate between a deterministic probability flow ODE, the original SDE, and a substantially more volatile sampler. We demonstrate that GraphBSI achieves state-of-the-art performance on the GuacaMol and Moses benchmarks for large molecule generation, outperforming existing models on nearly all metrics. Finally, in our ablations we empirically show that noise control critically influences performance.

\paragraph{Limitations and Future Work.}
GraphBSI, in its current implementation, suffers from the quadratic scaling of compute and memory requirements in the number of nodes that comes with the application of a graph transformer. Exploring a more memory-efficient graph neural network architecture to generate larger graphs would be a promising avenue for future research. Further, while GraphBSI allows for variable-sized graphs, the number of nodes is sampled beforehand instead of jointly generated with the graph features. Allowing for nodes to appear or disappear while generating the graph, similar to jump diffusion \citep{campbell2023transdimensionalgenerativemodelingjump}, might result in a more flexible generative process.

\section*{Acknowledgments}
This paper is supported by the DAAD programme Konrad Zuse Schools of Excellence in Artificial Intelligence, sponsored by the Federal Ministry of Research, Technology and Space.”

\bibliography{iclr2026_conference}
\bibliographystyle{iclr2026_conference}

\clearpage
\appendix

\section{Relationship to BFNs and Diffusion Models}
\label{app:bsi_bfn_equivalence}
\subsection{Relationship to BFNs}
There is a close equivalence between Categorical Bayesian Sample Inference (BSI) and Categorical Bayesian Flow Networks (BFNs). In fact, Categorical BFNs can be seen as a special case of Categorical BSI with a specific choice of prior distribution and noise schedule. The dynamics of BFNs are recovered when choosing the sampler in \cref{eq:generalized_sde} with $\gamma=1$ and $\beta_0=0$ to parametrize $\z_0$, i.e., making the prior logits deterministic. Note that we require $\beta_0>0$ to avoid numerical issues when approximating the score function. 
This generalized SDE allows BSI to vary stochasticity. Intuitively, increasing stochasticity allows the model to overwrite errors from previous predictions (see \cref{app:infinite_noise_sampling} for a discussion on the extreme case), and empirically, increasing stochasticity proves crucial for performance \cref{fig:noise_ablation}. To illustrate this, we will show the relationship between the components of both frameworks.

\textbf{Input Distribution} Both BFNs and categorical BSI parameterize the distribution over the data $\x$ using a categorical distribution. The logits are denoted as $\z$ in BSI and as $\theta$ in BFNs. In BSI, the parameters $\z$ are the logits of a categorical distribution, i.e., $p(\x \mid \z) \sim \cat(\softmax(\z))$. In BFNs, the parameters $\theta$ are the probabilities of each category, i.e., $p(\x \mid \theta) \sim \cat(\theta)$. The two parameterizations are equivalent since $\theta = \softmax(\z)$ and $\z = \log(\theta)$ (up to an additive constant).

\textbf{Output Distribution} The output distribution in BFNs is an intermediate distribution that is not needed in BSI.

\textbf{Prior Distribution} While Categorical BSI includes a normal prior distribution over the logits of the categorical distribution $(p(\z \mid t=0)\sim \mathcal{N}(\muzero, \beta_0 I))$, Categorical BFNs fix the parameters to $\theta_0=1/K$. Therefore, categorical BFNs can be seen as a special case of categorical BSI with $\muzero=0$ and $\beta_0 =0$.

\textbf{Sender Distribution} The sender distribution in categorical BFNs is an intermediate distribution that is not required in categorical BSI.

\textbf{Receiver Distribution} The sender distribution in categorical BFNs is given as

\[p_R(\y \mid \x,\alpha)\sim \sum_k \softmax(\Psi(\theta))_k \mathcal{N}(\alpha (K \hat{e}_k -1), \alpha K I)\]

It corresponds to the noisy measurement distribution in categorical BSI, $p(\y \mid \x,\alpha)\sim \mathcal{N}(\hat{\x}, 1/\alpha I)$. Note that for $\alpha \to 0$, it holds that:

\[p_R(\y \mid \x,\alpha)\sim \mathcal{N}(\alpha (K \softmax(\Psi(\theta)) -1), \alpha K I)\]

The sender distribution for $\alpha \to 0$ is an affine transformation of the noisy observation function for BSI: If we set $\y \sim p(\y \mid \x,\alpha)=\mathcal{N}(\hat{\x}, 1/\alpha I)$ and compute $y'=\alpha(K \y -1)$, then $y' \sim p_s(y' \mid \x,\alpha)$, where $\softmax(\Psi(\theta))$ corresponds to the sample reconstruction $\hat{\x}$. Thus, in the small-$\alpha$-limit, the two distributions have same-order approximation and therefore contain the same information. However, in the formulation of categorical BSI, we can directly see that $\y$ is a noisy observation of $\x$ and we do not require computing the distribution as a limit of a multinomial distribution as in BFNs.

\textbf{Bayesian Update Function} The Bayesian update function in categorical BFNs \cite[Eq. 171]{graves2025bayesianflownetworks} is the equivalent of \cref{theorem:update_equation_v2} in categorical BSI. The update is simplified for BSI since the belief parameters are in logit space instead of probability space. Furthermore, the scaling of the receiver distribution leads to an extra factor of $\alpha$ in categorical BSI.

\textbf{Bayesian Update Distribution} This is an intermediate that is not required in categorical BSI.

\textbf{Accuracy Schedule} The accuracy schedule can be chosen freely in categorical BSI. In categorical BFNs, the accuracy schedule is chosen as $\beta(t)=t^2\beta(1)$.

\textbf{Bayesian Flow Distribution} The Bayesian flow distribution in categorical BFNs corresponds to \cref{eq:zi_marginal} in categorical BSI. The two distributions are equivalent up to an affine transformation of the variable, as explained above.

\textbf{Continuous Time Loss} The continuous time loss in categorical BFNs \cite[Eq. 205]{graves2025bayesianflownetworks} corresponds to \cref{eq:categorical_bsi_loss} in categorical BSI. Both are the L2 loss between the reconstruction and the one-hot encoded data.

\textbf{SDE formulation} Both BSI and BFN sampling can be formulated as SDEs. Here, \cref{thm:categorical_bsi_sde} corresponds to \cite[Eq. 24]{xue2024unifyingbayesianflownetworks}. To do so, the authors also operate on the logits of the categorical distribution instead of the probabilities.

\textbf{Score function approximation} The score function approximation for categorical BFNs \cite[Eq. 28]{xue2024unifyingbayesianflownetworks} corresponds to \cref{theorem:score_function} for $\beta_0=0$ up to a constant. Note that a value of $\beta_0>0$ avoids the division by zero in the score function approximation at $t=0$.

\subsection{Relationship to diffusion models.}
    The logits $\z$ evolve in a way that closely resembles a diffusion process in logit space. From \cref{theorem:update_equation_v2} we have our denoising dynamics
\begin{equation}
    p(\z_{t+1}\mid \z_t, \x)
    = \mathcal{N}(\z_t + \alpha_t \x,\, \alpha_t \mathbf{I}).
\end{equation} 
Moreover, the marginal of $\z_t$ is given by
\begin{equation}
    p(\z_t \mid \x)
    = \mathcal{N}\bigl(\mu_0 + \beta(t)\x,\; (\beta_0 + \beta(t))\mathbf{I}\bigr)
    \label{eq:zi_marginal_repeat}
\end{equation}
(see \cref{eq:zi_marginal}). We define the corresponding ``noising'' process as the reverse-time conditional $p(\z_t \mid \z_{t+1}, \x)$. Using the standard Gaussian conditioning formula \citep[Eq.~4.125]{murphy2012machine}, we obtain
\begin{align}
    p(\z_t \mid \z_{t+1}, \x)
    = \mathcal{N}\!\left(
        \frac{(\beta_0 + \beta(t)) \z_{t+1} + \alpha_t \mu_0 - \alpha_t \beta_0 \x}
             {\beta_0 + \beta(t) + \alpha_t},
        \;
        \frac{\alpha_t (\beta_0 + \beta(t))}
             {\beta_0 + \beta(t) + \alpha_t}\,\mathbf{I}
    \right).
\end{align}
Thus, the reverse transition is Gaussian, analogous to the posterior 
$q(\x_{t-1} \mid \x_t, \x_0)$ in standard diffusion models. While this is not a 
typical diffusion process in the sense that the derived forward dynamics over 
$\z_t$ are generally non-Markovian, related non-Markovian formulations have 
been proposed before \citep{song2020denoising}. Interestingly, a Markovian process is 
recovered when setting $\beta_0 = 0$, which coincides with the original BFN 
parameterization \citep{graves2025bayesianflownetworks}.

\subsection{Relationship with Flow Matching Models}

At noise level $\gamma=0$, Categorical BSI is closely related to Flow Matching. The sampling SDE \cref{eq:generalized_sde} becomes an ODE where the right-hand side can be interpreted as an approximation of the flow field to follow. However, we do not train to directly predict the flow field, but to reconstruct the clean sample. Similar to Dirichlet Flow Matching (DFM), \cite{stark2024dirichletflowmatchingapplications}, Categorical BSI operates on a distribution over the simplex. However, while Categorical BSI uses the logits of a categorical distribution as a latent variable, DFM employs a mixture of Dirichlet distributions.

\section{BSI for sequence generation}
\label{sec:sequences}

Categorical BSI can generate general categorical data - it is not restricted to graphs. In this section, we demonstrate this capability empirically by training a categorical BSI model to generate sequences. We represent sequences with length $l$ and a vocabulary $v$ in the one-hot encoded format as $S\in\Delta^l_v \subset [0,1]^{l\times v}$. We call the resulting model SeqBSI.

Employing the same reconstruction model as \citealp{stark2024dirichletflowmatchingapplications,davis2024fisherflowmatchinggenerative}, a Convolutional Neural Network. We train on the toy dataset from \cite{davis2024fisherflowmatchinggenerative} with $l=4$ and $v\in\{5,10,20,40,60,80,100,120,140,160\}$ as well as a dataset of enhancer DNA sequences from fly brain cells \citet{flybrain} with $l=500$ and $v=4$ nucleotide bases.

Following \citet{stark2024dirichletflowmatchingapplications}, we report the KL divergence for the toy task and the Fréchet Biological Distance (FBD) as a measure of distribution similarity. As demonstrated in \cref{fig:sequence_results}, SeqBSI slightly outperforms Dirichlet Flow Matching \citep{stark2024dirichletflowmatchingapplications} in the flybrain task. The comparison with Fisher Flow Matching on this metric is difficult, as their evaluation shows vastly different results for Dirichlet flow matching than the results reported in their own paper. On the toy dataset task, SeqBSI outperforms Dirichlet Flow Matching and is competitive with Fisher Flow Matching (see \cref{fig:toy_experiment_results}).

\begin{table}
    \centering
    \caption{Results on the enhancer DNA sequence dataset}
    \includegraphics[width=\textwidth]{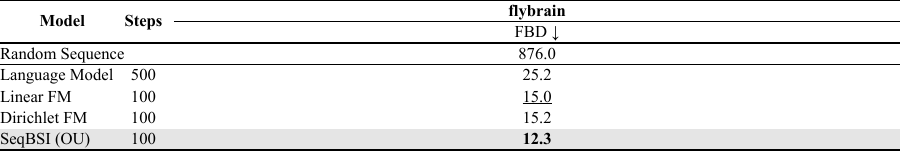}
    \label{fig:sequence_results}
\end{table}

\clearpage
\section{Analysis of SDE-based Sampling Algorithms}

In this section, we analyze the behavior of the SDE-based sampling methods \cref{alg:generalized_sde_sampling:euler_maruyama,alg:generalized_sde_sampling:ornstein_uhlenbeck}.

\subsection{Equivalence of the two Sampling Algorithms for Infinite Steps}
\label{app:equivalence_of_sampling_algorithms}
It is worth noting that for $\Delta t \to 0$, the Ornstein-Uhlenbeck discretization and the Euler-Maruyama discretization of \cref{eq:generalized_sde} converge to the same update step:

\begin{align}
    \z_{t + \Delta t} & \sim m + (\z_t - m) e^{-\kappa \Delta t} + \sqrt{\frac{\gamma \beta'}{2 \kappa}(1 - e^{-2 \kappa \Delta t})} \cdot \mathcal{N}(0, 1)                                         \\
                      & \to  m + (\z_t - m) (1-\kappa \Delta t) + \sqrt{\frac{\gamma \beta'}{2 \kappa}(1 - (1-2 \kappa \Delta t))} \cdot \mathcal{N}(0, 1)                                           \\
                      & = \z_t + \kappa(m - \z_t) \Delta t + \sqrt{\gamma \beta' \Delta t} \cdot \mathcal{N}(0, 1)                                                                                    \\
                      & = \z_t +  \kappa (\muzero+(\beta + \beta'/\kappa)\hat{\x} - \z_t) \Delta t + \sqrt{\gamma \beta' \Delta t} \cdot \mathcal{N}(0, 1)                                               \\
                      & = \z_t + \beta' \hat{\x} \Delta t + \frac{\gamma -1}{2}\beta' \frac{\muzero + \beta \hat{\x}-\z_t}{\beta + \beta_0} \Delta t+\sqrt{\gamma\beta' \Delta t} \cdot \mathcal{N}(0, 1) \\
                      & = \z_t + \beta' f_\theta(\z_t, t) \Delta t + \frac{\gamma -1}{2}\beta'\nabla_{\z_t}\log p_t(\z_t) \Delta t+\sqrt{\gamma\beta' \Delta t} \cdot \mathcal{N}(0, 1)
\end{align}

\subsection{Stability of Euler-Maruyama Sampling}\label{app:euler_maruyama_stability}
Let us explicitly write out the update step of the Euler-Maruyama discretization of \cref{eq:generalized_sde}:

\begin{align}
    \z_{t + \Delta t} & \sim \z_t + \beta' \hat{\x} \Delta t + \frac{\gamma -1}{2}\beta' \frac{\muzero + \beta \hat{\x}-\z_t}{\beta + \beta_0} \Delta t+\sqrt{\gamma\beta' \Delta t} \cdot \mathcal{N}(0, 1)                                                  \\
                      & = \left(1 - \frac{(\gamma -1)\beta'}{2(\beta + \beta_0)}\Delta t\right) \z_t + \beta' \hat{\x} \Delta t + \frac{(\gamma -1)\beta'(\muzero +\beta \hat{\x})}{2(\beta + \beta_0)} \Delta t + \sqrt{\gamma\beta' \Delta t} \cdot \mathcal{N}(0, 1)
\end{align}

As a rule of thumb, the coefficient in front of $\z_t$ should not be negative, i.e., the previous step should not be over-corrected. This yields the condition
\begin{align}
    1 - \frac{(\gamma -1)\beta'}{2(\beta + \beta_0)}\Delta t & \geq 0                                 \\
    \iff \Delta t \cdot (\gamma -1)                          & \leq \frac{2(\beta + \beta_0)}{\beta'}
\end{align}
For our precision schedule on moses ($\beta_\mathtt{start}=3.0, \beta_\mathtt{end}=12.0, \beta_0=1.0$), we find that

\begin{equation}\label{eq:stability_condition}
    \min_{t\in[0,1]} \frac{2(\beta(t) + \beta_0)}{\beta'(t)} \approx 0.48
\end{equation}
The resulting maximum stable noise level $\gamma$ for different numbers of sampling steps in \cref{tab:stability_condition} predicts the observed behavior in \cref{fig:noise_ablation} surprisingly well.

\begin{table}[h]
    \centering
    \caption{Maximum stable $\gamma$ for different numbers of sampling steps with the Euler-Maruyama discretization, following \cref{eq:stability_condition}.}
    \label{tab:stability_condition}
    \begin{tabular}{rrr}
        \toprule
        Number of Timesteps & $\Delta t$ & Maximum Stable $\gamma$ \\
        \midrule
        25                  & 0.040000   & 12.938480               \\
        50                  & 0.020000   & 24.876960               \\
        100                 & 0.010000   & 48.753920               \\
        200                 & 0.005000   & 96.507840               \\
        500                 & 0.002000   & 239.769601              \\
        \bottomrule
    \end{tabular}
\end{table}

\subsection{Behavior of Ornstein-Uhlenbeck Sampling with Infinite Noise}
\label{app:infinite_noise_sampling}
Taking the limit $\gamma \to \infty$ in \cref{alg:generalized_sde_sampling:ornstein_uhlenbeck} yields an interesting sampling algorithm (see \cref{alg:bsi_sampling_categorical:gamma_to_inf}). In this limit, the update step becomes independent of the previous step $\z_t$, replacing all previous information with the current prediction $\hat{\x}$. Empirically, we find that fixing the prior value after the initial sampling step, as shown in \cref{alg:bsi_sampling_categorical:gamma_to_inf:fixed_prior}, works better in practice (see \cref{tab:molecule_results_extended}). This algorithm matches the Flowback algorithm from \citet{song2025smooth}. We find that with a budget of 50 sampling steps, this algorithm performs surprisingly well on molecule generation. However, a higher compute budget drastically reduces performance. We hypothesize that this is because an excessive amount of stochasticity is introduced.
\citet{song2025smooth} address this by adaptively alternating between vanilla BFN steps and Flowback steps, effectively mixing \cref{alg:bsi_sampling_categorical} with \cref{alg:bsi_sampling_categorical:gamma_to_inf}.

\begin{minipage}{0.48\textwidth}
    \begin{algorithm}[H]
        \caption{Sampling with $\gamma \to \infty$}
        \label{alg:bsi_sampling_categorical:gamma_to_inf}
        \begin{algorithmic}
            \Require reconstructor $f_\theta$, discretization $\Delta t$, precision schedule $\beta: [0, 1] \to \mathbb{R}^+$
            \State $\z_0 \sim \mathcal{N}(\muzero, \beta_0 I)$
            \State $\z \gets \z_0$
            \For{$t = 0 \dots 1$ in steps of $\Delta t$}
            \State $\hat{\x} \gets f_\theta(\z, t)$
            \State $\alpha \gets \beta_0 + \beta(t + \Delta t/2)$
            \State $\y \sim \mathcal{N}(\mu=\hat{\x}, \Sigma^2=1/\alpha \cdot I)$
            \State $\triangleright$ Go from prior to $t$ in single step
            \State $\z \gets \muzero + \alpha \cdot \y$
            \EndFor
            \State \Return $\mathrm{Quantize}(f_\theta(\z, 1))$
        \end{algorithmic}
    \end{algorithm}
\end{minipage}
\hfill
\begin{minipage}{0.48\textwidth}
    \begin{algorithm}[H]
        \caption{Fixed-prior sampling with $\gamma \to \infty$}
        \label{alg:bsi_sampling_categorical:gamma_to_inf:fixed_prior}
        \begin{algorithmic}
            \Require reconstructor $f_\theta$, discretization $\Delta t$, precision schedule $\beta: [0, 1] \to \mathbb{R}^+$
            \State $\z_0 \sim \mathcal{N}(\muzero, \beta_0 I)$
            \State $\z \gets \z_0$
            \For{$t = 0 \dots 1$ in steps of $\Delta t$}
            \State $\hat{\x} \gets f_\theta(\z, t)$
            \State $\alpha \gets \beta(t + \Delta t/2)$
            \State $\y \sim \mathcal{N}(\mu=\hat{\x}, \Sigma^2=1/\alpha \cdot I)$
            \State $\triangleright$ Go from prior to $t$ in single step
            \State $\z \gets \z_0 + \alpha \cdot \y$
            \EndFor
            \State \Return $\mathrm{Quantize}(f_\theta(\z, 1))$
        \end{algorithmic}
    \end{algorithm}
\end{minipage}

\begin{table}[h]
  \centering
  \caption{Results on the GuacaMol and Moses benchmarks for molecular generation with the Euler- (EM) and Ornstein-Uhlenbeck (OU) discretization, and with \cref{alg:bsi_sampling_categorical:gamma_to_inf} ($\gamma \to \infty$) and \cref{alg:bsi_sampling_categorical:gamma_to_inf:fixed_prior} ($\gamma \to \infty$, FP), as well as results for a linear scheduler (lin) with the same final precision as the exponential scheduler. Additionally, we include results obtained with the FlowBack (FB) sampler \citet{song2025smooth} using a smaller value of $\beta_0$, as well as the OU sampler with the same checkpoint. The EM sampler becomes unstable at $\beta_0=0.05$.}\label{tab:molecule_results_extended}
  \includegraphics[width=\textwidth]{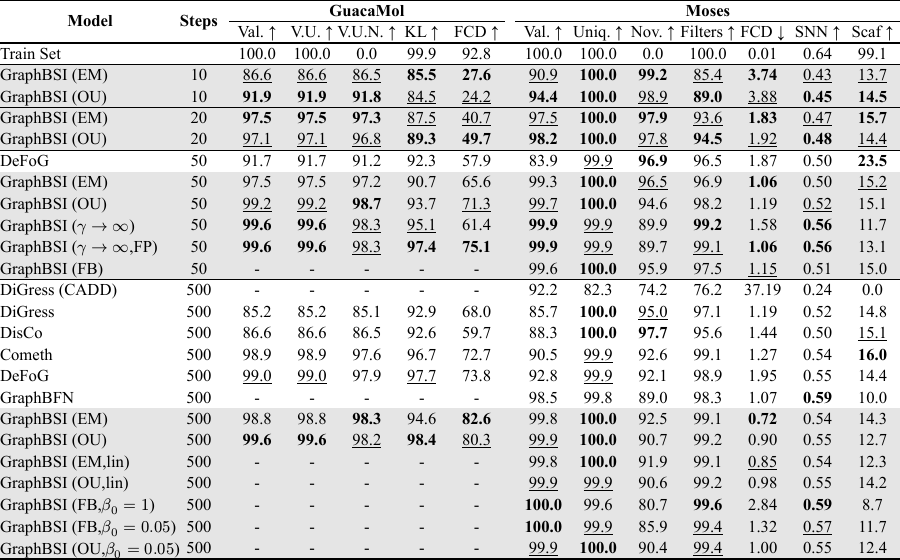}
\end{table}

\begin{figure}
    \centering
    \includegraphics[width=0.5\linewidth]{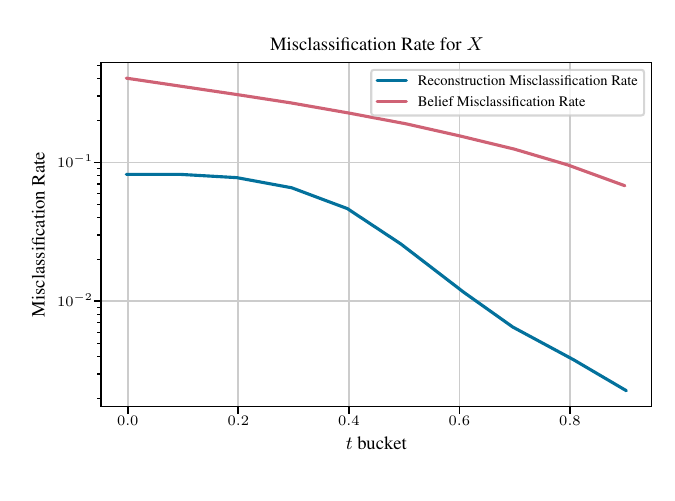}
    \caption{Empirical misclassification rate of a trained reconstructor on the moses dataset under the encoding distribution. Compared to simply sampling from the belief, returning a reconstruction is far more likely to yield the correct train sample. Therefore, returning a quantization of the reconstruction instead of sampling from the belief is significantly more efficient for molecule generation. However, deriving the ELBO under quantization is intractable to optimize. Therefore, we have the sampling-formulation to derive a tractable ELBO and the quantized-formulation to optimize efficiency after training.}
    \label{fig:missclassification_rate_by_t}
\end{figure}

\begin{figure}[ht]
    \centering

    \begin{subfigure}{0.8\linewidth}
        \centering
        \includegraphics[width=\linewidth]{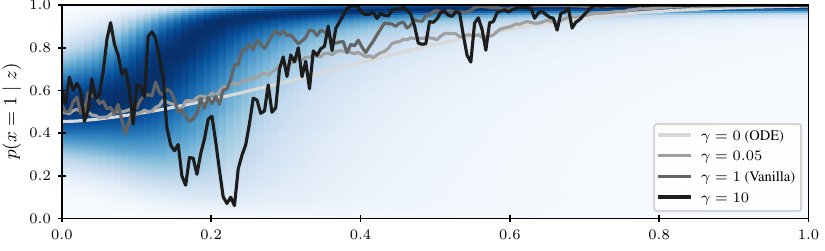}
        \caption{Exemplary trajectories}
        \label{fig:binary_sampler:exemplary_trajectories}
    \end{subfigure}

    \vspace{1em} %

    \begin{subfigure}{0.8\linewidth}
        \centering
        \includegraphics[width=\linewidth]{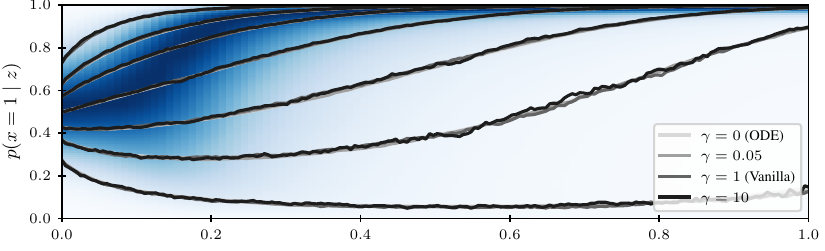}
        \caption{Quantiles ($1\%,10\%,50\%,75\%,90\%,99\%$) over $10000$ trajectories}
        \label{fig:binary_sampler:quantiles}
    \end{subfigure}

    \caption{Illustration of the trajectories of the categorical sampler with two categories with a fixed reconstruction $f(\z,t)=\hat{e}_1$ for different noise levels $\gamma$. While higher values of $\gamma$ result in more volatile trajectories (see \cref{fig:binary_sampler:exemplary_trajectories}), the marginal distribution is preserved if the score function is known exactly (see \cref{fig:binary_sampler:quantiles}). Since we approximate the score function in practice, the noise level is a crucial hyperparameter to fin-tune during inference.}
    \label{fig:binary_sampler}
\end{figure}

\begin{figure}
    \centering
    \includegraphics[width=0.9\linewidth]{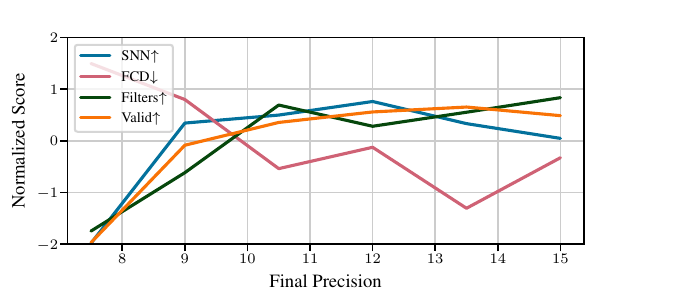}
    \caption{Key metrics on the Moses benchmark with a linear scheduler, ending at different final precisions. The model was trained with a final precision of 15, and to generate this plot, sampling was stopped early instead of training a new model for each precision value. While too small final precision values yield noisy samples, too large final precision values waste sampling steps.}
    \label{fig:final_precision_ablations}
\end{figure}
\begin{figure}
    \centering
    \includegraphics[width=\textwidth]{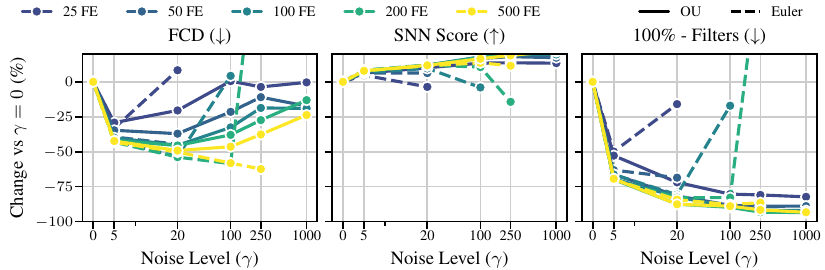}
    \caption{Change in metrics relative to $\gamma=0$ vs. noise level $\gamma$ for different numbers of function evaluations (FE) and discretization schemes. Our custom Ornstein-Uhlenbeck discretization scheme is denoted as OU, while the standard Euler-Maruyama scheme is written as Euler. Some values for the Euler scheme are missing since the sampler becomes unstable if $\gamma \cdot \Delta t$ becomes too large (see \cref{app:euler_maruyama_stability}).}
    \label{fig:noise_ablation_relative}
\end{figure}
\begin{figure}
    \centering
    \includegraphics[width=0.8\linewidth]{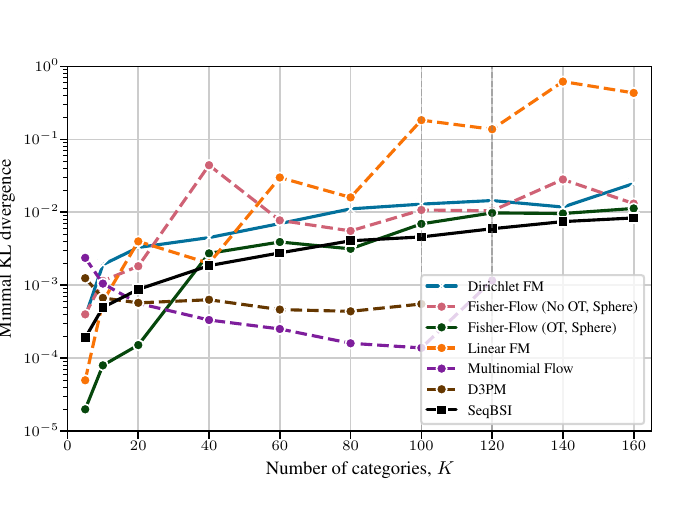}
    \caption{KL divergence on the toy sequences benchmark by \citep{davis2024fisherflowmatchinggenerative}, reporting the lowest KL divergence for each vocabulary size over five random seeds. The model is trained on $100,000$ samples with a sequence of length four and varying vocabulary size. Find the details of the dataset generation in the original paper.}
    \label{fig:toy_experiment_results}
\end{figure}
\clearpage

\section{Proofs}\label{app:proofs}
\newtheorem*{T2}{Theorem~\ref{theorem:update_equation_v2}}
\begin{T2}
    Given a prior belief $p(\x\mid \z)=\cat(\x\mid \softmax(\z))$, after observing $\y \sim \mathcal{N}(\y \mid  \mu = \x, \Sigma^2 = 1/\alpha I)$ at precision $\alpha$, the posterior belief is $p(\x\mid \z,\y,\alpha) =\cat(\x\mid \softmax(\z_\mathrm{post}))$ with
    \begin{equation}
        \z_\mathrm{post} = \z + \alpha \cdot \y
    \end{equation}
\end{T2}
\begin{proof}
    We need to compute the Bayesian update of the belief parameters. Each dimension can be considered independently since the noise is isotropic. Let us start with a single-variable prior belief $\cat(\softmax(\z))$ with $\z \in \mathbb{R}^c$, and a noisy observation  $\y \mid \x, \alpha \sim \mathcal{N}(\mu = \x, \Sigma ^2 = 1/\alpha \cdot I)$ of the true sample $\x \in \Delta^{c-1}$ at precision $\alpha$. Let us now consider any class $l \in {1,\dots,c}$. We write $\hat{e}_l$ for the one-hot encoding of class $l$. Since we are only interested in the ratio of the posterior probabilities, we can ignore any factors that do not depend on $l$ and normalize at the end. We have:
    \begin{align}
        p(\x=\hat{e}_l \mid \z)            & = \softmax(\z)_l \propto \exp(\z_{l})                                                             \\
        p(\y \mid \x=\hat{e}_l, \alpha)    & = \mathcal{N}(\y |
        \mu =\hat{e}_l, \Sigma ^2 = 1/\alpha \cdot I)                                                                                           \\
        p(\x=\hat{e}_l \mid \z, \y, \alpha) & = \propto p(\y \mid \x=\hat{e}_l, \alpha) \cdot p(\x=\hat{e}_l \mid \z)                                           \\
                                      & = \mathcal{N}(\y \mid \mu =\hat{e}_l, \Sigma ^2 = 1/\alpha \cdot I) \cdot \softmax(\z)_l             \\
                                      & \propto \exp\left(-\frac{||\y - \hat{e}_l||^2}{2 \cdot 1/\alpha}\right) \cdot \exp(\z_{l})                \\
                                      & = \exp\left(-\frac{||\y||^2 - 2 \cdot \langle \y, \hat{e}_l\rangle + ||\hat{e}_l||^2}{2 \cdot 1/\alpha} + \z_{l}\right) \\
                                      & \propto \exp\left(\alpha \cdot \y_{l} + \z_{l}\right)
    \end{align}

    Let us now normalize the results to obtain the posterior probabilities:
    \begin{align}
        p(\x=\hat{e}_l \mid \z, \y, \alpha) & = \frac{\exp\left(\alpha \cdot \y_{l} + \z_{l}\right)}{\sum_{l'=1}^c \exp\left(\alpha \cdot \y_{l'} + \z_{l'}\right)} = \softmax(\z + \alpha \cdot \y)_l
    \end{align}
    Putting everything together, we find that the posterior belief is $p(\x\mid \z,\y,\alpha) =\cat(\x\mid \softmax(\z_\mathrm{post}))$ with
    \begin{equation}
        \z_\mathrm{post} = \z + \alpha \cdot \y
    \end{equation}
\end{proof}

\newtheorem*{T3}{Theorem~\ref{theorem:bsi_elbo}}
\begin{T3}
    For categorical BSI, the log-likelihood of $\x$ under \cref{alg:bsi_sampling_categorical} is lower-bounded by
    \begin{equation}
        \log p(\x) \geq \EE\limits_{\substack{\z_k \sim q(\z|\x,t_k)}}[\log p(\x\mid\z_k)] -\frac{k}{2}  \EE\limits_{\substack{i \sim \mathcal{U}(0,k-1) \\ \z_i \sim q(\z|\x,t_i)}}[(\beta(t_{i+1})-\beta(t_i))||f_\theta(\z_i,t_i)-\x||_2 ^2],
    \end{equation}
    where $q(\z \mid \x, t) = \mathcal{N}(\z \mid \muzero + \beta(t)\x, \beta_0 + \beta(t) I)$.
\end{T3}
\begin{proof}
    For any distribution $p(\x)$ and any latent variable $\z$, i.e. any choice of prior $p(\z)$, encoding distribution $p(\z \mid \x)$, and likelihood $p(\x \mid \z)$, we have the variational lower bound
    \begin{equation}
        \log p(\x) \geq \EE_{\z \sim p(\z \mid \x)}[\log p(\x \mid \z)] - \mathrm{KL}(p(\z \mid \x) \| p(\z))
    \end{equation}
    on $\log p(\x)$ \cite{kingma2022autoencodingvariationalbayes}. We choose the beliefs $\z_0, \dots, \z_k$ as latent variables at the discretized time steps $t_0, \dots, t_k$. We choose the encoding distribution to be the distribution of the beliefs under \cref{alg:bsi_sampling_categorical} with the reconstruction network $f_\theta$ replaced by the true sample $\x$:
    \begin{equation}
        p(\z_0, \dots, \z_k \mid \x) = \mathcal{N}(\z_0 \mid \muzero, \beta_0 I) \prod_{i=0}^{k-1} p(\z_{i+1} \mid \z_i, \x, t_i)
    \end{equation}
    The transition distribution $p(\z_{i+1} \mid \z_i, \x, t_i)$ can be computed from \cref{theorem:update_equation_v2}:
    \begin{equation}
        \z_{i+1} = \z_i + \alpha_i \cdot \y_i \sim \z_i + \alpha_i \cdot \mathcal{N}(\y \mid \mu = \x, 1/\alpha_i I) = \mathcal{N}(\z_{i+1} \mid \z_i + \alpha_i \cdot \x, \alpha_i I)
    \end{equation}

    The distribution of $p(\z)$ following \cref{alg:bsi_sampling_categorical} factorizes similarly:
    \begin{equation}
        p(\z_0, \dots, \z_k) = \mathcal{N}(\z_0 \mid \muzero, \beta_0 I) \prod_{i=0}^{k-1} p(\z_{i+1} \mid \z_i, t_i, \theta)
    \end{equation}
    with the transition distribution
    \begin{equation}
        p(\z_{i+1} \mid \z_i, t_i, \theta) = \mathcal{N}(\z_{i+1} \mid \z_i + \alpha_i \cdot f_\theta(\z_i, t_i), \alpha_i I)
    \end{equation}
    Let us now compute the KL divergence:
    \begin{align}
         & \mathrm{KL}(p(\z_0, \dots, \z_k \mid \x) \| p(\z_0, \dots, \z_k))                                                                                                                                                                 \\
         & = \EE_{\substack{\z_0, \dots, \z_k \sim                                                                                                                                                                                     \\ p(\z_0, \dots, \z_k \mid \x)}}\left[\log \frac{p(\z_0, \dots, \z_k \mid \x)}{p(\z_0, \dots, \z_k)}\right] \\
         & = \EE_{\substack{\z_0, \dots, \z_k \sim                                                                                                                                                                                     \\ p(\z_0, \dots, \z_k \mid \x)}}\left[\log \frac{\mathcal{N}(\z_0 \mid \muzero, \beta_0 I) \prod_{i=0}^{k-1} p(\z_{i+1} \mid \z_i, \x, t_i)}{\mathcal{N}(\z_0 \mid \muzero, \beta_0 I) \prod_{i=0}^{k-1} p(\z_{i+1} \mid \z_i, t_i, \theta)}\right] \\
         & = \EE_{\substack{\z_0, \dots, \z_k \sim                                                                                                                                                                                     \\ p(\z_0, \dots, \z_k \mid \x)}}\left[\sum_{i=0}^{k-1} \log \frac{p(\z_{i+1} \mid \z_i, \x, t_i)}{p(\z_{i+1} \mid \z_i, t_i, \theta)}\right] \\
         & = \sum_{i=0}^{k-1} \EE_{\substack{\z_i \sim  p(\z_i \mid \x)}}\left[\mathrm{KL}(p(\z_{i+1} \mid \z_i, \x, t_i) \| p(\z_{i+1} \mid \z_i, t_i, \theta))\right]                                                                               \\
         & = \sum_{i=0}^{k-1} \EE_{\substack{\z_i \sim  p(\z_i \mid \x)}}\left[\mathrm{KL}(\mathcal{N}(\z_{i+1} \mid \z_i + \alpha_i \cdot \x, \alpha_i I) \| \mathcal{N}(\z_{i+1} \mid \z_i + \alpha_i \cdot f_\theta(\z_i, t_i), \alpha_i I))\right] \\
         & = \sum_{i=0}^{k-1} \EE_{\substack{\z_i \sim  p(\z_i \mid \x)}}\left[\frac{1}{2\alpha_i} ||\z_i + \alpha_i \cdot \x - (\z_i + \alpha_i \cdot f_\theta(\z_i, t_i))||_2^2\right]                                                       \\
         & = \sum_{i=0}^{k-1} \EE_{\substack{\z_i \sim  p(\z_i \mid \x)}}\left[\frac{\alpha_i}{2} ||\x - f_\theta(\z_i, t_i)||_2^2\right]                                                                                                    \\
         & = \sum_{i=0}^{k-1} \EE_{\substack{\z_i \sim  p(\z_i \mid \x)}}\left[(\beta(t_{i+1})-\beta(t_i))/2||\x - f_\theta(\z_i, t_i)||_2^2\right]                                                                                            \\
         & = \EE_{\substack{i \sim \mathcal{U}(0,k-1)                                                                                                                                                                                \\ \z_i \sim p(\z_i \mid \x)}}\left[\frac{k}{2} (\beta(t_{i+1})-\beta(t_i))||\x - f_\theta(\z_i, t_i)||_2^2\right]
    \end{align}
    Since $p(\x \mid \z_0, \dots, \z_k) = p(\x \mid \z_k) = \cat(\x \mid \softmax(\z_k))$, we can plug in \cref{eq:zi_marginal} to obtain the final result:

    \begin{equation}
        \log p(\x) \geq \EE\limits_{\substack{\z_k \sim q(\z|\x,t_k)}}[\log p(\x\mid\z_k)] -\frac{k}{2}  \EE\limits_{\substack{i \sim \mathcal{U}(0,k-1) \\ \z_i \sim q(\z|\x,t_i)}}[(\beta(t_{i+1})-\beta(t_i))||f_\theta(\z_i,t_i)-\x||_2 ^2],
    \end{equation}
    where $q(\z \mid \x, t) = \mathcal{N}(\z \mid \muzero + \beta(t)\x, \beta_0 + \beta(t) I)$.
\end{proof}

\newtheorem*{TSDE}{Theorem~\ref{theorem:sde}}
\begin{TSDE}
    As $\Delta t \to 0$, the update equation in \cref{theorem:update_equation_v2} converges to the following SDE:
    \begin{align}
        d\z_t & = \beta'(t) f_\theta(\z_t, t) dt + \sqrt{\beta'(t)} dW_t  %
    \end{align}
    where $dW_t$ is a Wiener process and $\z_{0}\sim \mathcal{N}(\muzero, \beta_0 \cdot I)$.
\end{TSDE}

\begin{proof}
    Take the update equation \cref{theorem:update_equation_v2} with an infinitesimal time step $\Delta t \to 0$, it holds that

    \begin{equation}
        \alpha = (\beta(t+\Delta t)-\beta(t)) \to \beta'(t) \Delta t
    \end{equation}
    Therefore, we have:
    \begin{align}
        \z_{t+\Delta t} &= \z_t + \alpha \y \\
                        &\sim \z_t + \alpha \mathcal{N}(\hat{\x}, \Sigma^2 = 1/\alpha I) \\
                        &= \z_t + \mathcal{N}(\alpha \hat{\x}, \Sigma^2 = \alpha I) \\
                        &\to \z_t + \beta'(t) \hat{\x}\Delta t + \sqrt{\beta'(t)}\sqrt{\Delta t}\cdot \mathcal{N}(0,\mathbf{I})
    \end{align}
    We identify this as the Euler-Maruyama discretization of the SDE above.
\end{proof}

\newtheorem*{T1}{Theorem~\ref{theorem:generalized_sde}}
\begin{T1}
    The SDE in \cref{theorem:sde} is generalized by the following family of SDEs with equal marginal densities $p_t(\z_t)$:
    \begin{align}
        d\z_t & = \beta'(t) f_\theta(\z_t, t) dt + \frac{\gamma -1}{2}\beta'(t)\nabla_{\z_t}\log p_t(\z_t)dt+\sqrt{\gamma\beta'(t)} dW_t
    \end{align}
    where $dW_t$ is a Wiener process and $\z_{0}\sim p(\z\mid t=0)$.
\end{T1}
\begin{proof}
    We need to show that the evolution of the probability density $p_t(\z_t)$ of \cref{eq:categorical_bsi_sde} matches that of \cref{eq:generalized_sde}. The evolution is characterized by the Fokker-Planck equation:
    \begin{align*}
         & \frac{\partial p_t(\z_t)}{\partial t} = \sum_j - \nabla_{\z_j}\left ( \beta'(t) f_\theta(\z_t, t) + \frac{\gamma -1}{2}\beta'(t)\nabla_{\z_t}\log p_t(\z_t) \right )p_t(\z_t) + \frac{1}{2}\gamma\beta'(t)\nabla_{\z_j}^2p_t(\z_t) \\
         & = \sum_j - \nabla_{\z_j}\left ( \beta'(t) f_\theta(\z_t, t)p_t(\z_t) \right ) - \frac{\gamma -1}{2}\beta'(t)\nabla_{\z_j} \left(p_t(\z_t)\nabla_{\z_j}\log p_t(\z_t)\right) + \frac{1}{2}\gamma\beta'(t)\nabla_{\z_j}^2p_t(\z_t)    \\
         & = \sum_j - \nabla_{\z_j}\left ( \beta'(t) f_\theta(\z_t, t)p_t(\z_t) \right ) - \frac{\gamma -1}{2}\beta'(t)\nabla_{\z_j}^2  p_t(\z_t) + \frac{1}{2}\gamma\beta'(t)\nabla_{\z_j}^2p_t(\z_t)                                       \\
         & = \sum_j - \nabla_{\z_j}\left ( \beta'(t) f_\theta(\z_t, t)p_t(\z_t) \right ) + \frac{1}{2}\beta'(t)\nabla_{\z_j}^2p_t(\z_t)
    \end{align*}
    Which equals the Fokker-Planck equation of the SDE in \cref{eq:categorical_bsi_sde}.
\end{proof}

\newtheorem*{T4}{Theorem~\ref{theorem:score_function}}
\begin{T4}
    The BSI loss \cref{eq:categorical_bsi_loss} also is a score matching loss with the score model $s_\theta(\z,t)$ parameterized as
    \begin{equation}
    \label{eq:score_function_appendix}
        s_\theta(\z, t) \equiv \frac{\muzero + \beta(t)f_{\theta}(\z,t)-\z}{\beta(t) + \beta_0}  \overset{!}{\approx} \nabla_{\z} \log p_t(\z)
    \end{equation}
\end{T4}

\begin{proof}
    Score matching \cite{song2021scorebasedgenerativemodelingstochastic} is a generative model that learns to approximate the score function $\nabla_{\z} \log p_t(\z)$ of a distribution $p_t(\z)$ by minimizing the score matching loss:
    \begin{equation}
        \mathcal{L}_\mathtt{score} \equiv \mathbb{E}_{t \sim \mathcal{U}(0,1)}[\lambda(t) \mathbb{E}_{p(\x)} \mathbb{E}_{p_t(\z \mid \x)} \left[ \|s_\theta(\z, t) - \nabla_{\z} \log p_t(\z \mid \x)\|_2^2 \right]]
    \end{equation}
    where $\lambda: [0,1] \mapsto \mathbb{R}^+$ is a positive weighting function. The distribution $p_t(\z \mid \x)$ is the distribution of the latent variable at time $t$ given the true sample $\x$.
    For categorical BSI, we have from \cref{eq:zi_marginal}:
    \begin{equation}
        p_t(\z \mid \x) = \mathcal{N}(\z \mid \muzero + \beta(t)\x, (\beta_0 + \beta(t)) I)
    \end{equation}

    The score function of an isotropic Gaussian can be computed in closed form:
    \begin{align}
        \nabla_{\z} \log \mathcal{N}(\z \mid \mu, \sigma^2 I) & = \nabla_{\z} \left(-\frac{||\z - \mu||^2}{2\sigma^2}\right) = -\frac{\z - \mu}{\sigma^2} \\
    \end{align}
    Plugging in the parameters of $p_t(\z \mid \x)$, we find:
    \begin{align}
        \nabla_{\z} \log p_t(\z \mid \x) & = -\frac{\z - (\muzero + \beta(t)\x)}{\beta_0 + \beta(t)} = \frac{\muzero + \beta(t)\x - \z}{\beta_0 + \beta(t)}
    \end{align}
    With the proposed score model parameterization $s_\theta(\z,t)$, we find:
    \begin{align}
         & \mathcal{L}_\mathtt{score} = \mathbb{E}_{t \sim \mathcal{U}(0,1)}[\lambda(t) \mathbb{E}_{p(\x)} \mathbb{E}_{p_t(\z \mid \x)} \left[ \|s_\theta(\z, t) - \nabla_{\z} \log p_t(\z \mid \x)\|_2^2 \right]]                                                       \\
         & = \mathbb{E}_{t \sim \mathcal{U}(0,1)}[\lambda(t) \mathbb{E}_{p(\x)} \mathbb{E}_{p_t(\z \mid \x)} \left[ \left\|\frac{\muzero + \beta(t)f_{\theta}(\z,t)-\z}{\beta(t) + \beta_0} - \frac{\muzero + \beta(t)\x - \z}{\beta_0 + \beta(t)}\right\|_2^2 \right]] \\
         & = \mathbb{E}_{t \sim \mathcal{U}(0,1)}[\lambda(t) \mathbb{E}_{p(\x)} \mathbb{E}_{p_t(\z \mid \x)} \left[ \left\|\frac{\beta(t)(f_{\theta}(\z,t)-\x)}{\beta(t) + \beta_0}\right\|_2^2 \right]]       \\
         &= \mathbb{E}_{t \sim \mathcal{U}(0,1)}[\lambda(t)\frac{\beta(t)^2}{(\beta(t) + \beta_0)^2} \mathbb{E}_{p(\x)} \mathbb{E}_{p_t(\z \mid \x)} \left[ \left\|(f_{\theta}(\z,t)-\x)\right\|_2^2 \right]]       \\
    \end{align}
    Choosing the weighting
    \begin{equation}
        \lambda(t) = \beta'(t) \frac{(\beta(t) + \beta_0)^2}{2\beta(t)^2},
    \end{equation}
    \label{eq:score_weighting}
    we find that the score matching loss equals the BSI loss in \cref{eq:categorical_bsi_loss}. 
Therefore, the BSI loss in \cref{eq:categorical_bsi_loss} is a score-matching loss with the weighting \cref{eq:score_weighting} and the score function $s_\theta(\z, t)$ parameterized as in \cref{eq:score_function_appendix}.
    
\end{proof}

\newtheorem*{T5}{Theorem~\ref{theorem:discretized_sde}}

\begin{T5}
    Fixing the prediction $\hat{\x}=f_\theta(\z_t, t)$ and the values $\beta = \beta(t+\Delta t / 2)$, $\beta' = \beta'(t+\Delta t / 2)$ in \cref{eq:generalized_sde}  in a time interval $[t, t+\Delta t]$ yields an Ornstein-Uhlenbeck process with the exact marginal

    \begin{equation}
        \z_{t + \Delta t} \sim m + (\z_t - m) e^{-\kappa \Delta t} + \sqrt{\frac{\gamma \beta'}{2 \kappa}(1 - e^{-2 \kappa \Delta t})} \cdot \mathcal{N}(0, I),
    \end{equation}
    where $\kappa = \frac{(\gamma-1)\beta'}{2(\beta_0+\beta)}$, $m=\muzero+(\beta + \beta'/\kappa)\hat{\x}$.
\end{T5}

\begin{proof}
    The SDE in \cref{eq:generalized_sde} with fixed parameters $\beta, \beta', \hat{\x}$ is given as
    \begin{align}
        d\z_t & = \beta' \hat{\x} dt + \frac{\gamma -1}{2}\beta'\nabla_{\z_t}\log p_t(\z_t)dt+\sqrt{\gamma\beta'} dW_t
    \end{align}
    where $dW_t$ is a Wiener process and $\z_{t}\sim p(\z\mid t)$. Let us insert \cref{theorem:score_function} to obtain
    \begin{align}
        d\z_t & = \beta' \hat{\x} dt + \frac{\gamma -1}{2}\beta' \frac{\muzero + \beta f_{\theta}(\z_t,t)-\z_t}{\beta + \beta_0} dt+\sqrt{\gamma\beta'} dW_t \\
              & = \frac{(\gamma-1)\beta'}{2(\beta_0+\beta)}\left(\muzero + \left(\beta + \frac{2(\beta_0+\beta)}{\gamma-1}\right)\hat{\x}-\z_t\right) dt+\sqrt{\gamma\beta'} dW_t
    \end{align}
    Setting $\kappa = \frac{(\gamma-1)\beta'}{2(\beta_0+\beta)}$ and $m=\muzero+(\beta + \beta'/\kappa)\hat{\x}$, we find
    \begin{equation}
        d\z_t = \kappa(m - \z_t) dt + \sqrt{\gamma \beta'} dW_t
    \end{equation}
    which is an Ornstein-Uhlenbeck process. The exact marginal distribution of an Ornstein-Uhlenbeck process is given as \cite{uhlenbeck1930theory}:
    \begin{equation}
        \z_{t + \Delta t} \sim m + (\z_t - m) e^{-\kappa \Delta t} + \sqrt{\frac{\gamma \beta'}{2 \kappa}(1 - e^{-2 \kappa \Delta t})} \cdot \mathcal{N}(0, I)
    \end{equation}
\end{proof}

\section{Additional Results}\label{app:extra_results}

\begin{table}
    \centering
    \caption{Results on the QM9 dataset.}
    \includegraphics[width=\textwidth]{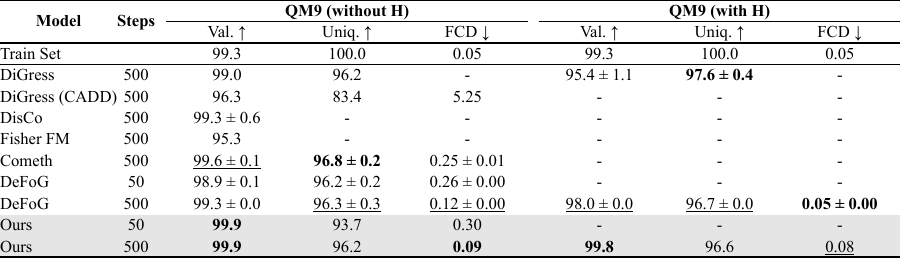}
    \label{fig:qm9_results}
\end{table}

\cref{fig:qm9_results} shows our method is competitive on the QM9 dataset with removed hydrogen atoms, achieving state-of-the-art results on validity and FCD. We explicitly model charges on the nodes, enabling high validity scores.

\begin{table}
    \centering
    \caption{Hyperparameters used for the results in \cref{tab:molecule_results,tab:synthetic_graph_results}. The precision schedule is parameterized as $\beta(t)=\beta_\mathtt{start}\cdot (\exp(t \cdot \log(\beta_\mathtt{end}/\beta_\mathtt{start}))-1)$.}\label{tab:hyperparameters}
    \includegraphics[width=\textwidth]{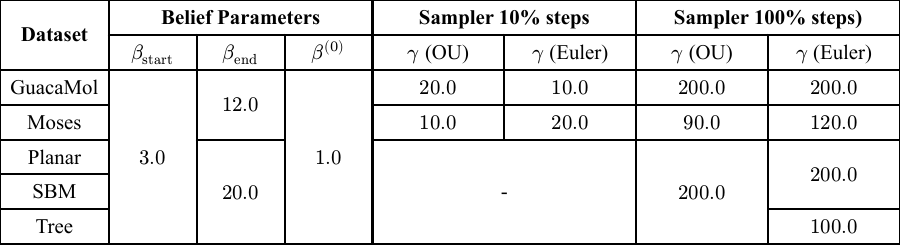}
\end{table}

\begin{table}
\centering
\caption{Datasets with training samples and maximum number of nodes. For Moses, we use the \texttt {test\_scaffolds} split for benchmarking, which is the standard test split.}
\label{tab:datasets}
\begin{tabular}{lcc}
\toprule
\textbf{Dataset} & \textbf{Train samples} & \textbf{Max. Nodes} \\
\midrule
GuacaMol~\citep{guacamol} & 1.3M & 88 \\
Moses~\citep{moses} & 1.6M & 30 \\
Planar~\citep{martinkus2022spectrespectralconditioninghelps} & 128 & 64 \\
SBM~\citep{martinkus2022spectrespectralconditioninghelps} & 128 & 187 \\
Tree~\citep{bergmeister2024efficientscalablegraphgeneration} & 128 & 64 \\
\bottomrule
\end{tabular}
\end{table}

\begin{table}
    \centering
    \caption{Molecular metrics}\label{tab:molecular_metrics}
    \includegraphics[width=\textwidth]{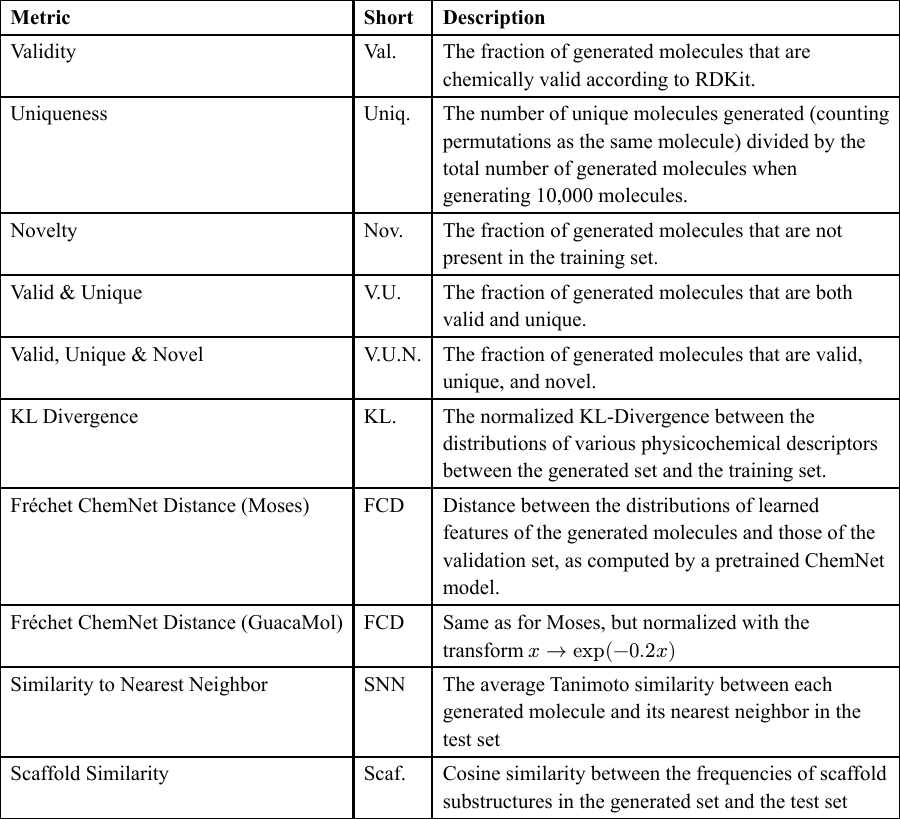}
\end{table}

\begin{table}
    \centering
    \caption{Synthetic graph metrics metrics}\label{tab:synthetic_graph_metrics}
    \includegraphics[width=\textwidth]{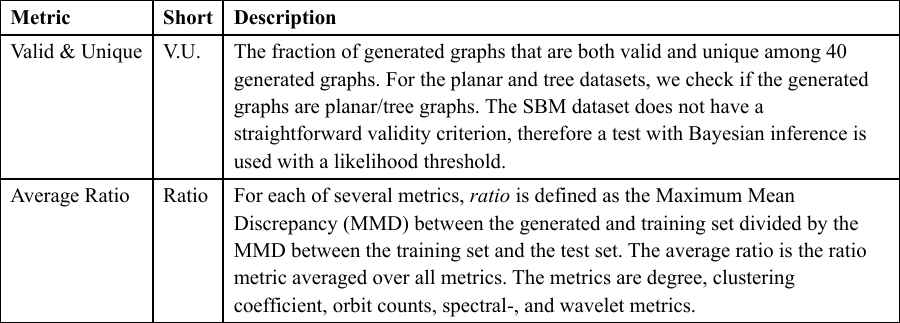}
\end{table}

\end{document}